\documentclass[12pt]{article}
\usepackage[T1]{fontenc}
\usepackage[utf8]{inputenc}
\usepackage{graphicx,psfrag,epsf}
\usepackage{enumerate,enumitem}
\usepackage{natbib}
\usepackage{url}
\usepackage{booktabs}
\usepackage[bookmarks,bookmarksnumbered,%
    allbordercolors={0.8 0.8 0.8},%
    linktocpage%
    ]{hyperref} 
\hypersetup{colorlinks,
  citecolor={blue},linkcolor={black}}
\usepackage{textcomp}
\usepackage{times}
\usepackage{amssymb,amsthm,amsmath,mathrsfs,bbm,nicefrac}
\usepackage{color}
\usepackage{subfig}
\usepackage{caption}
\usepackage{algorithm}
\usepackage{algpseudocode}
\usepackage{multirow}
\usepackage{makecell}
\usepackage{ulem}
\usepackage[toc,page,title,titletoc,header]{appendix}

\newcommand{\blind}{1}

\newtheorem{thm}{Theorem}[section]
\newtheorem{lemma}{Lemma}[section]

\newtheorem{defini}{Definition}[section]

\usepackage{amsmath,amsfonts,bm}

\DeclareMathAlphabet{\mathsfit}{\encodingdefault}{\sfdefault}{m}{sl}
\SetMathAlphabet{\mathsfit}{bold}{\encodingdefault}{\sfdefault}{bx}{n}

\def\gP{{\mathcal{P}}}

\DeclareMathOperator*{\argmax}{arg\,max}
\DeclareMathOperator*{\argmin}{arg\,min}
\DeclareMathOperator*{\arginf}{arg\,inf}
\DeclareMathOperator*{\argsup}{arg\,sup}

\newcommand{\bmu}{\mbox{\boldmath $\mu$}}
\newcommand{\btheta}{\mbox{\boldmath $\theta$}}

\newcommand{\bw}{\mathbf{w}}

\newcommand{\baa}{{\bf a}}

\newcommand{\GG}{\mathbb{G}}

\newcommand\raiseT[2]{%
  \setbox0\hbox{$#1{#2}$}\raise\dp0\box0}

\newcommand{\bea}{\begin{eqnarray*}}
\newcommand{\eea}{\end{eqnarray*}}
\newcommand{\ba}{\begin{eqnarray*}}
\newcommand{\ea}{\end{eqnarray*}}
\newcommand{\be}{\begin{equation}}
\newcommand{\ee}{\end{equation}}
\newcommand{\bi}{\begin{itemize}}
\newcommand{\ei}{\end{itemize}}

\newcommand{\cd}{\overset{d}{\longrightarrow}}

\newcommand{\Xcal}{\mathcal{X}}

\begin{document}

\def\spacingset#1{\renewcommand{\baselinestretch}%
{#1}\small\normalsize} \spacingset{1}

\if1\blind
{
  \title{\bf Minimum Wasserstein Distance Estimator under Finite Location-scale Mixtures}
  \author{Qiong Zhang\thanks{
    Contact: Qiong Zhang (qiong.zhang@stat.ubc.ca) and Jiahua Chen (jhchen@stat.ubc.ca). 
    Department of Statistics, 3182 Earth Sciences Building, 
    2207 Main Mall Vancouver, BC Canada V6T 1Z4.}
    \hspace{.2cm}
    and Jiahua Chen \\
    Department of Statistics, University of British Columbia}
  \maketitle
} \fi

\if0\blind
{
  \bigskip
  \begin{center}
    {\LARGE\bf Minimum Wasserstein Distance Estimator under Finite Location-scale Mixtures
}
\end{center}
  \medskip
} \fi

\medskip
\begin{abstract}
When a population exhibits heterogeneity, we often model it via a finite mixture: decompose it into several different but homogeneous subpopulations. 
Contemporary practice favors learning the mixtures by maximizing the likelihood for statistical efficiency and the convenient EM-algorithm for numerical computation.
Yet the maximum likelihood estimate (MLE) is not well defined for the most widely used finite normal mixture in particular and for finite location-scale mixture in general.
We hence investigate feasible alternatives to MLE such as minimum distance estimators.
Recently, the Wasserstein distance has drawn increased attention in the machine learning community.
It has intuitive geometric interpretation and is successfully employed in many new applications.
Do we gain anything by learning finite location-scale mixtures via a minimum Wasserstein distance estimator (MWDE)?
This paper investigates this possibility in several respects.
We find that the MWDE is consistent and derive a numerical solution under finite location-scale mixtures.
We study its robustness against outliers and mild model mis-specifications.
Our moderate scaled simulation study shows the MWDE suffers some efficiency loss against a penalized version of MLE in general without noticeable gain in robustness.
We reaffirm the general superiority of the likelihood based learning strategies even for the non-regular finite location-scale mixtures.
\end{abstract}

\noindent%
\vfill

\newpage
\spacingset{1.45} 

\section{Introduction}
Let $\mathcal{F} = \{ f(\cdot |\btheta): \btheta \in \Theta\}$ be a parametric distribution family with density function $f(\cdot | \btheta)$ with respect to some $\sigma$-finite measure. 
Denote by $G = \sum_{k=1}^K w_k \{\btheta_k\}$  a distribution assigning probability $w_k$ on $ \btheta_k \in \Theta$.
A distribution with the following density function
\begin{equation*}
\label{def:general_mixture_density}
    f(x | G) = \int f(x|\btheta)dG(\btheta) = \sum_{k=1}^K w_k f(x|\btheta_k)
\end{equation*}
is called a finite $\mathcal{F}$ mixture.
We call $f(x|\btheta)$ the subpopulation density function, 
$\btheta$ the subpopulation parameter, and $w_k$ the mixing weight of the $k$th subpopulation. 
We use $F(x|\btheta)$ and $F(x | G)$ for the cumulative distribution functions (CDF) of $f(x|\btheta)$ and $f(x|G)$ respectively. 
Let
\begin{equation*}
\mathbb{G}_{K}
=
\big \{
G: G =\sum_{k=1}^K  w_k \{\btheta_k\},  0 \leq w_k \leq 1,
\sum_{k=1}^{K} w_k=1, \btheta_k \in \Theta 
\big \}
\end{equation*}
be a space of mixing distributions with at most $K$ support points.
A mixture distribution of (exactly) order $K$ has its mixing distribution
$G$ being a member of $\mathbb{G}_{K}-\mathbb{G}_{K-1}$. 

We study the problem of learning the mixing distribution $G$ given a set of independent and identically distributed (IID) observations $\Xcal = \{x_1,x_2,\ldots, x_N\}$ from a mixture $f(x | G)$. 
Throughout the paper, we assume the order of $G$ is known
and $\mathcal{F}$ is a known location-scale family. 
That is,
\begin{equation*}
f(x | \btheta) = \frac{1}{\sigma} f_0 \Big (\frac{x-\mu}{\sigma} \Big )
\end{equation*}
for some probability density function $f_0(x)$ 
with $ x \in \mathbb{R}$ with respect to Lebesgue measure where
$\btheta = (\mu, \sigma)$ with $\Theta = \{\mathbb{R} \times \mathbb{R}^{+}\}$.

Finite mixture models provide a natural representation of heterogeneous population that is believed to be composed of several homogeneous subpopulations ~\citep{pearson1894contributions,schork1996mixture}. 
They are also useful for approximating distributions with unknown shapes which are particularly relevant in image generation~\citep{kolouri2017sliced}, image segmentation~\citep{farnoosh2008image}, object tracking~\citep{santosh2013tracking}, and signal processing~\citep{kostantinos2000gaussian}.

In statistics, the most fundamental task is to learn the unknown parameters.
In early days, the method of moments was the choice for its ease of computation~\citep{pearson1894contributions} under finite mixture models.
Nowadays, the maximum likelihood estimate (MLE) is the first choice 
due to its statistical efficiency and the availability of an easy-to-use EM-algorithm. 
Under a finite location-scale mixture model, the log-likelihood function of $G$ is given by
\begin{equation}
\label{eq:log_likelihood}
\ell_N(G|\mathcal{X}) 
= 
\sum_{n=1}^{N} \log f(x_n| G) 
= 
\sum_{n=1}^{N} \log 
\Big  \{
\sum_{k=1}^K \frac{w_k}{\sigma_k}
f_0
\Big (\frac{x_{n}-\mu_k}{\sigma_k} \Big ) \Big \}.
\end{equation}
At an arbitrary mixing distribution
\(
G_\epsilon = 0.5 \{(x_1, \epsilon)\} + 0.5 \{(0, 1)\},
\)
we have
$\ell_N(G_\epsilon|\mathcal{X}) \to \infty$ as $\epsilon \to 0$.
Hence, the MLE of $G$ is not well defined or is ill defined.
Various remedies, such as penalized maximum likelihood estimate (pMLE), 
has been proposed to overcome this obstacle~\citep{chen2008inference,chen2009inference}.
At the same time, MLE can be thought of a special minimum distance estimator. 
It minimizes a specific Kullback-Leibler divergence between the empirical distribution 
and the assumed model ${\cal F}$.
Other divergences and distances have been investigated in the literature as in
 \citet{choi1969estimators, yakowitz1969consistent, woodward1984comparison, 
clarke1994robust, cutler1996minimum, deely1968construction}.
Recently, the Wasserstein distance has drawn increased attention in 
machine learning community due to its intuitive interpretation and good geometric properties~\citep{evans2012phylogenetic,arjovsky2017wasserstein}.
The Wasserstein distance based estimator for learning finite mixture models
is absent in the literature.

Are there any benefits to learn finite location-scale mixtures by the minimum Wasserstein distance estimator (MWDE)?
This paper answers this question from several angles.
We find that the MWDE is consistent and derive a numerical solution under finite location-scale mixtures.
We compare the robustness of the MWDE with pMLE in the presence of outliers and mild model mis-specifications.
We conclude that the MWDE suffers some efficiency loss against pMLE in general without obvious gain in robustness.
Through this paper, we better understand the pros and cons of the 
MWDE under finite location-scale mixtures.
We reaffirm the general superiority of the likelihood based
learning strategies even for the non-regular finite location-scale mixtures.

In the next section, we first introduce the Wasserstein distance and
some of its properties.
This is followed by a formal definition of the MWDE, a discussion of its existence 
and consistency under finite location-scale mixtures. 
In Section~\ref{sec:WDE_numerical_computation}, we give some algebraic results 
that are essential for computing $2$-Wasserstein distance 
between the empirical distribution and the finite location-scale mixtures.
We then develop a BFGS algorithm scheme for computing the MWDE of the mixing distribution. 
In addition, we briefly review the penalized likelihood approach and its numerical issues.
In Section~\ref{sec:exp}, we characterize the efficiency properties of the 
MWDE relative to pMLE in various circumstances via simulation. 
We also study their robustness when the data contains outliers, is contaminated
or when the model is mis-specified.
We then apply both methods in an image segmentation example.
We conclude the paper with a summary in Section~\ref{sec:conclusion}.

\section{Wasserstein Distance and the Minimum Distance Estimator}
\label{sec:preliminary}
\subsection{Wasserstein Distance}
Wasserstein distance is a distance between probability measures.
Let $\Omega$ be a Polish space endowed with a ground distance $D(\cdot,\cdot)$
and $\gP(\Omega)$ the space of Borel probability measures on $\Omega$.
Let $\eta \in \gP(\Omega)$ be a probability measure.
If for some $p  > 0$,
\[
\int_\Omega D^p(x,  x_0) \eta (dx) < \infty,
\]
for some (and thus any) $x_0 \in \Omega$, we say $\eta$ has finite $p$th moment.
Denote by $\gP_p(\Omega) \subset \gP(\Omega)$ the space of probability measures 
with finite $p$th moment. 
For any $\eta, \nu \in \gP(\Omega)$, we use $\Pi(\eta, \nu)$ to denote the space of the bivariate probability measures on $\Omega \times \Omega$ whose marginals are $\eta$ and $\nu$.
Namely, 
\[
 \Pi(\eta,\nu) =\{\pi \in \gP(\Omega^2):
    \int_{\Omega} \pi(x,dy) = \eta(x),~\int_{\Omega} \pi(dx,y) = \nu(y) \}.
\]
The $p$-Wasserstein distance is defined as follows.

\begin{defini}[$p$-Wasserstein distance]\label{def:pWasserstein}
For any $\eta, \nu \in \gP_p(\Omega)$ with $p \geq 1$,
the $p$th Wasserstein distance between $\eta$ and $\nu$ is
\begin{equation*}
\label{eq:p-WD}
    W_{p}(\eta ,\nu )
    =\Big \{ \inf _{\pi \in \Pi (\eta ,\nu )}\int _{\Omega^2}D^p(x, y) \pi (dx, dy)
      \Big \}^{1/p}.
\end{equation*}
\end{defini}

Suppose  $X$ and $Y$ are two random variables whose distributions are $F$ and $G$ 
and induced probability measures are $\eta$ and $\nu$.
We regard the $p$-Wasserstein distance between $\eta$ and $\nu$ also the distance between random variables or distributions: $W_p(X, Y)= W_p(F, G) = W_p(\eta, \nu)$.

The $p$-Wasserstein distance is a distance on $\gP_p(\Omega)$ as shown by ~\citet[Theorem 7.3]{villani2003topics}.
For any $\eta, \nu, \xi \in \gP_p(\Omega)$, it has the following properties:
\begin{enumerate}[label={(\arabic*)}]
    \item 
    Non-negativity: $W_p(\eta, \nu) \geq 0$ and $W_p(\eta, \nu)=0$ if and only if $\eta=\nu$;
    \item 
    Symmetry: $W_p(\eta, \nu) = W_p(\nu, \eta)$;
    \item 
    Triangular inequality: $W_p(\eta, \nu) \leq W_p(\eta, \xi) + W_p(\xi, \nu)$.
\end{enumerate}

The Wasserstein distance has many nice properties. 
Let us denote $\eta_n \cd \eta$ for convergence in distribution or measure.
\citet[Theorem 7.1.2]{villani2003topics} shows that it has the following properties:
\begin{itemize}
\item[]
{\bf Property 1}. For any $q \geq p \geq 1$, $W_q(\eta,\nu)\geq W_p(\eta,\nu)$.
\item[]
{\bf Property 2}.
$W_p(\eta_n, \eta) \to 0$ as $n \to \infty$ if and only if both
\begin{itemize}
\item[(i)] 
$\eta_n \cd \eta$, and
\item[(ii)] 
$\int D^p(x, x_0) \eta_n(dx) \to \int D^p(x, x_0) \eta(dx)$ 
for some (and thus any) $x_0  \in \Omega$.
\end{itemize}
\end{itemize}

Computing the Wasserstein distance involves a challenging optimization problem in general but has a simple solution under a special case.
Suppose $\Omega$ is the space of real numbers, $D(x, y)=| x - y |$, and $F$ and $G$ are univariate distributions.
Let $F^{-1}(t):=\inf\{x :F(x) \geq t\}$ and $G^{-1}(t):=\inf\{x :G(x) \geq t\}$ for $t \in [0, 1]$ be their quantile functions.
We can easily compute the Wasserstein distance based on the following property.
\begin{itemize}
\item[]
{\bf Property 3}.
$
W_p(F , G )= \big \{ \int_{0}^1 |F^{-1}(t)-G^{-1}(t)|^p dt \big \}^{1/p}
$.
\end{itemize}

\subsection{Minimum Wasserstein Distance Estimator}

Let $W_p(\cdot,\cdot)$ be the $p$-Wasserstein distance with ground distance 
$D(x, y) = | x - y|$ for univariate random variables.
Let $\Xcal = \{x_1,x_2,\ldots, x_N\}$ be a set of IID observations
 from finite location-scale mixture $f(x | G)$ of order $K$ 
 and $F_N(x) = N^{-1}\sum_{n=1}^N \mathbbm{1}(x_n\leq x)$ be the empirical distribution.
We introduce the MWDE of the mixing distribution $G$ that is
\begin{equation}
\label{obj:WDE_GMM}
\hat{G}_N^{\text{MWDE}} =
\arginf_{G\in \GG_K} 
W_p(F_N(\cdot), F(\cdot|G))
=
\arginf_{G\in \GG_K} W_p^p(F_N(\cdot), F(\cdot|G)).
\end{equation}

As we pointed out earlier, the MLE is not well defined under finite location-scale mixtures. 
Is the MWDE well defined? We examine the existence or sensibility of the MWDE. 
We show that the MWDE exists when $f_0(\cdot)$ satisfies certain conditions.

Assume that $f_0(0) > 0$, $f_0(x)$ is bounded, continuous, and has finite $p$th moment.
Under these conditions, we can see
\[
0 \leq W_p(F_N(\cdot), F(\cdot|G)) < \infty
\]
for any $G \in \GG_K$.
When $N \leq K$, the solution to~\eqref{obj:WDE_GMM} merits special attention.
Let $G_\epsilon = \sum_{n=1}^N (1/N) \{(x_n, \epsilon)\}$ be a mixing distribution 
assigning probability $1/N$ on $\btheta_n=(x_n,\epsilon)$.
When $\epsilon \to 0$, each subpopulation in the mixture $f(x| G_\epsilon)$ 
degenerates to a point mass at $x_n$.
Hence, as $\epsilon \to 0$,
\[
W_p(F_N(\cdot), F(\cdot|G_\epsilon)) \to 0.
\]
Since none of $G \in \GG_K$ has zero-distance from $F_N(\cdot)$, 
the MWDE does not exist unless we expand $\GG_K$ to include 
$G_0 = \sum_{n=1}^N (1/N) \{(x_n, 0)\} = \lim G_\epsilon$.
To remove this technical artifact, in the MWDE definition we expand the space of 
$\sigma$ to $[0, \infty)$.
We denote by $F(\cdot | (\theta_0, 0))$ a distribution with point mass at $x = \theta_0$. 
With this expansion, $G_0$ is the MWDE when $N \leq K$.

Let 
\(
\delta = \inf\{ W_p(F_N(\cdot), F(\cdot|G)): G \in \GG_K\}
\).
Clearly, $0 \leq \delta < \infty$.
By definition, there exists a sequence of mixing distributions 
$G_m \in \GG_K$ such that $W_p(F_N(\cdot), F(\cdot|G_m))  \to \delta$ as $m \to \infty$. 
Suppose one mixing weight of $G_m$ has limit 0.
Removing this support point and rescaling, we get a new mixing distribution sequence 
and it still satisfies $W_p(F_N(\cdot), F(\cdot|G_m))  \to \delta$.
For this reason, we assume that its mixing weights have non-zero limits by selecting converging subsequence if necessary to ensure the limits exist.
Further, when the mixing weights of $G_m$  assume their limiting values 
while keeping subpopulation parameters the same, we still have 
$W_p(F_N(\cdot), F(\cdot|G_m))  \to \delta$ as $m \to \infty$.
In the following discussion, we therefore discuss the sequence of mixing distributions 
whose mixing weights are fixed.

Suppose the first subpopulation of $G_m$ has its scale parameter $\sigma_1 \to \infty$ 
as $m \to \infty$.
With the boundedness assumption on $f_0(x)$,
the mass of this subpopulation will spread thinly over entire ${\mathbb R}$
because $\sigma_1^{-1} f_0( (x - \mu_1)/\sigma_1) \to 0$ uniformly. 
For any fixed finite interval, [$a, b$], this thinning makes
\[
F(b | \btheta_1) - F(a | \btheta_1) \to 0
\]
as $m \to \infty$.
It implies that for any given $t \in (0, 0.5)$, we have
\[
|F^{-1} ( t | \btheta_1)| + |F^{-1} ( 1- t | \btheta_1)| \to \infty.
\]
This further implies for any $t \in (0, w_1/2)$, we have
\[
|F^{-1}(t | G_m)| + | F^{-1}(1 - t| G_m)| \to \infty
\]
as $m \to \infty$.
In comparison, the empirical quantile satisfies $x_{(1)} \leq F_N^{-1} (t) \leq x_{(N)}$ for any $t$.
By Property 3 of $W_p(\cdot, \cdot)$, these lead to 
$W_p(F_N(\cdot), F(\cdot|G_m))  \to \infty$ as $m \to \infty$.
This contradicts the assumption $W_p(F_N(\cdot), F(\cdot|G_m))  \to \delta$.
Hence, $\sigma_1 \to \infty$ is not a possible scenario of $G_m$ nor $\sigma_k \to \infty$ for any $k$.

Can a subpopulation of $G_m$ instead have its location parameter $\mu \to \infty$? 
For definitiveness, let this subpopulation correspond to $\btheta_1$.
Note that at least $w_1\{1- F_0(0)\}$-sized probability mass of $F(x| G_m)$ 
is contained in the range  $[\mu_1, \infty)$. 
Because of this, when $\mu_1 \to \infty$, we have 
$F^{-1}(1- t | G_m) \to \infty$ for $t = w_1\{1- F_0(0)\}/2$.
Therefore,  $W_p(F_N(\cdot), F(\cdot|G_m)) \to \infty$ by Property 3.
This contradicts $W_p(F_N(\cdot), F(\cdot|G_m))  \to \delta < \infty$.
Hence, $\mu_1 \to \infty$ is not a possible scenario of $G_m$ either.
For the same reason, we cannot have $\mu_k \to  \pm \infty$ for any $k$.

After ruling out  $\mu_k \pm \infty$ and $\sigma_k \to \infty$, 
we find $G_m$ has a converging subsequence 
whose limit is a proper mixing distribution in $\GG_K$. 
This limit is then an MWDE and the existence is verified.

The MWDE may not be unique and the mixing distribution may lead to a mixture with degenerate subpopulations.
We will show that the MWDE is consistent as the sample size goes to infinity. 
Thus, having degenerated subpopulations in the learned mixture is a mathematical artifact and also a sensible solution. 
In contrast, no matter how large the sample size becomes, there are always degenerated mixing distributions with unbounded likelihood values.

\subsection{Consistency of MWDE}
\label{sec:consistency}
We consider the problem when $\mathcal{X}=\{x_1, \ldots, x_N\}$
are IID observations from a finite location-scale mixture of order
$K$. The true mixing distribution is denoted as $G^*$.
Assume that $f_0(x)$ is bounded, continuous, and has finite $p$th moment.
We say the location-scale mixture is identifiable if
\[
F(x| G_1) = F(x| G_2)
\]
for all $x$ given $G_1, G_2 \in \GG_K$ implies $G_1 = G_2$.
We allow subpopulation scale $\sigma = 0$.
The most commonly used finite locate-scale mixtures, such as the normal mixture, are well known to be identifiable~\citep{teicher1961identifiability}. 
\citet{holzmann2004identifiability} give
a sufficient condition for the identifiability of general finite location-scale mixtures.
Let $\varphi(\cdot)$ be the characteristic function of $f_0(t)$.
The finite location-scale mixture is identifiable if for any $\sigma_1>\sigma_2>0$.
$\lim_{t\rightarrow \infty} \varphi(\sigma_1 t)/\varphi(\sigma_2 t)=0$.

We consider the MWDE based on $p$-Wasserstein distance with ground distance
$D(x, y) = |x - y|$ for some $p \geq 1$.
The MWDE under finite location-scale mixture model
as defined in~\eqref{obj:WDE_GMM} is asymptotically consistent.

\begin{thm}
\label{thm:consistency}
With the same conditions on the finite location-scale mixture and same notations above, we have the following conclusions.
\begin{enumerate}
\item
For any sequence $G_m \in \GG_K$ and  $G^* \in \GG_K$,
$W_p(F(\cdot|G_m), F(\cdot|G^*)) \to 0$ implies $G_m \cd G^*$
as $m \to \infty$.

\item
The MWDE satisfies
$W_p(F(\cdot|G^*), F(\cdot|\hat{G}_N^{\text{MWDE}})) \to 0$
as $N \to \infty$ almost surely.

\item
The MWDE is consistent: $W_p(\hat G_N^{\text{MWDE}}, G^*) \to 0$ 
as $N \to \infty$ almost surely.
\end{enumerate}
\end{thm}

\begin{proof}
We present these three conclusions in the current order which is easy to understand.
For the sake of proof, a different order is better.
For ease presentation, we write 
$F^* =  F(\cdot| G^*)$ and $\hat G = \hat G_N^{\text{MWDE}}$
in this proof.

We first prove the second conclusion. By the triangular inequality
and the definition of the  minimum distance estimator, we have
\[
W_p(F^*, F(\cdot|\hat{G}_N))
\leq
W_p(F_N, F^*)
+
W_p(F_N, F(\cdot|\hat{G}_N))
\leq
2 W_p(F_N, F^*).
\]
Note that $F_N$ is the empirical distribution and $F^*$ is the
true distribution, we have $F_N(x) \to F^*(x)$ uniformly in $x$
almost surely. At the same time, under the assumption that
$F_0(x)$ has finite $p$th moment, $F^*(x)$ also has finite $p$th
moment. The  $p$th moment of $F_N(x)$ converges to that of
$F^*(x)$ almost surely.
Given the ground distance $D(x, y) = |x - y|$, 
the $p$th moment in Wasserstein distance sense
is the usual moments in probability theory. 
By Property 2, we conclude $W_p(F_N, F(\cdot| G^*)) \to 0$
as both conditions there are satisfied.

Conclusion 3 is implied by Conclusions 1 and 2. With Conclusion 2
already established, we need only prove Conclusion 1 to complete
the whole proof. 
By Helly's lemma~\citep[Lemma 2.5]{van2000asymptotic}
again,  $G_m$ has a converging subsequence though the limit
can be a sub-probability measure.
Without loss of generality, we assume that $G_m$ itself converges
with limit $\tilde G$.
If $\tilde G$ is a sub-probability measure, so would be $F(\cdot | \tilde G)$.
This will lead to
\[
W_p(F(\cdot|G_m), F(\cdot|G^*)) 
\to 
W_p(F(\cdot|\tilde G), F(\cdot|G^*)) 
\neq 0
\]
which violates the theorem condition.
If $\tilde G$ is a proper distribution in $\GG_K$ and
\[
W_p(F(\cdot|\tilde G), F(\cdot|G^*))  = 0,
\]
then by identifiability condition, we have $\tilde G = G^*$.
This implies $G_m \to G^*$ and completes the proof.
\end{proof}

The multivariate normal mixture is another type of location-scale mixture.
The above consistency result of MWDE can be easily extended to finite 
multivariate normal mixtures.

\begin{thm}
\label{thm:NormalMixtureconsistency}
Consider the problem when $\mathcal{X}=\{x_1, \ldots, x_N\}$
are IID observations from a finite multivariate normal mixture distribution
of order $K$ and $\hat G_N^{\text{MWDE}}$ is the minimum Wasserstein
distance estimator defined by \eqref{obj:WDE_GMM}. 
Let the true mixing distribution be $G^*$.
The MWDE is consistent: $W_p(\hat G_N^{\text{MWDE}}, G^*) \to 0$  
as $N \to \infty$ almost surely.
\end{thm}

The rigorous proof is long though the conclusion is obvious.
We offer a less formal proof based on several well known probability theory results:
\begin{enumerate}
\item[(I)] A multivariate random variable sequence $Y_n$ converges
in distribution to $Y$ if and only if $\baa^\tau Y_n$ converges
to $\baa^\tau Y$ for any unit vector $\baa$;
\item[(II)] If $Y$ is multivariate normal if and only if $\baa^\tau Y$ is normal for all $\baa$;
\item[(III)] The normal distribution has finite moment of any order.
\end{enumerate}

Let $X_m$ be a random vector with distribution $F(\cdot | G_m)$ for some $G_m \in \GG_K$, $m=0, 1, 2, \ldots$, in a general mixture model setting.
Suppose as $m \to \infty$, with the notation we introduced previously,
\[
W_p(X_m, X_0) \to 0.
\]
Then for any unit vector $\baa$, based on property 2 of the Wasserstein distance and the result (I), we can see that
\[
W_p(\baa^\tau X_m, \baa^\tau X_0) \to 0.
\]
Next, we apply this result to normal mixture so that $F(\cdot | G_m)$ becomes $\Phi(\cdot |G_m)$ which stands for a finite multivariate normal mixture with mixing distribution $G_m$.
In this case, $X_m$ is a random vector with distribution $\Phi(\cdot |G_m)$. 
Let $(\bmu_k, \Sigma_k)$ be generic subpopulation parameters.
We can see that the distribution of $\baa^\tau X_m$, $\Phi_\baa (\cdot |G_m)$ is a finite normal mixture with subpopulation parameters $(\baa^\tau \bmu_k, \baa^\tau \Sigma_k \baa)$, and mixing weights the same as those of $G_m$.
Let the mixing distributions after projection be $G_{m, \baa}$ and $G_{0,\baa}$.

By the same argument in the proof of Theorem \ref{thm:consistency},
\[
W_p(\Phi (\cdot |\hat G_N), \Phi (\cdot |G^*)) \to 0
\]
almost surely as $N\to \infty$.
This implies
\[
W_p(\Phi_\baa (\cdot |\hat G_N), \Phi_\baa(\cdot |G^*)) \to 0
\]
almost surely as $N\to \infty$ for any $\baa$.
Hence, by Conclusion 1 of Theorem \ref{thm:consistency}, $\hat G_{N, \baa} \cd \hat G^*_\baa$
almost surely for any unit vector $\baa$.
We therefore conclude the consistency result: $\hat G_{N} \cd \hat G^*$ almost surely.

\subsection{Numerical Solution to MWDE}
\label{sec:WDE_numerical_computation}
Both in applications and in simulation experiments, we need an effective way to compute the MWDE.
We develop an algorithm that leverages the explicit form of the Wasserstein distance between two measures on $\mathbb{R}$ for the numerical solution to the MWDE.
The strategy works for any $p$-Wasserstein distance
but we only provide specifics for $p=2$ as it is the most widely used. 

Let $Y$ be a random variable with distribution $ f_0(\cdot)$.
Denote  the mean and variance of $Y$ by $\mu_0=\mathbb{E}(Y)$ and 
$\sigma_0^2=\text{Var}(Y)$.
Recall that $G = \sum_{k=1}^K w_k \{ (\mu_k, \sigma_k)\}$.
Let $x_{(1)}\leq x_{(2)}\leq\cdots\leq x_{(N)}$ be the order statistics,
$\overline{x^2} = N^{-1} \sum_{n=1}^N x_n^2$,
and
$\xi_n = F^{-1}(n/N| G)$ be the $(n/N)$th quantile of the mixture for $n=0, 1, \ldots, N$. 
Let 
\begin{equation*}
T(x) = \int_{-\infty}^{x} tf_0(t)dt
\end{equation*}
and define
\begin{align*}
\Delta F_{nk} 
&=
 F_0\Big (\frac{\xi_{n}-\mu_k}{\sigma_k}\Big) 
 - F_0\Big (\frac{\xi_{n-1}-\mu_k}{\sigma_k}\Big),
 \\
 \Delta T_{nk}
 &=
  T\Big (\frac{\xi_{n}-\mu_k}{\sigma_k}\Big) 
 - T\Big (\frac{\xi_{n-1}-\mu_k}{\sigma_k}\Big).
\end{align*}

When $p=2$, we expand the squared $W_2$ distance, $\mathbb{W}_N$,
between the empirical distribution and  $F(\cdot | G)$ as follows:
\begin{eqnarray*}
\label{eq:W2_GMM_obj}
\mathbb{W}_N(G) 
&=&
W_2^2(F_N(\cdot), F(\cdot|G)) \\
&=&
 \int_{0}^{1} \{ F_N^{-1}(t) - F^{-1}(t|G) \}^2 dt\\
&=&\overline{x^2} + 
\sum_{k=1}^K w_k \{\mu_k^2 + \sigma_k^2(\mu_0^2+\sigma_0^2)+2\mu_k\sigma_k\mu_0\}\\
&&- 2 \sum_k w_k 
\big \{
\mu_k \sum_{n=1}^{N} x_{(n)} \Delta F_{nk}
+
\sigma_k \sum_{n=1}^{N} x_{(n)}\Delta T_{nk}
\big \}.
\end{eqnarray*}
 
The MWDE minimizes $\mathbb{W}_N(G)$ with respect to $G$. 
The mixing weights and subpopulation scale parameters in this
optimization problem have natural constraints.
We may replace the optimization problem with an unconstrained one
by the following parameter transformation:
\begin{align*}
&\sigma_k = \exp(\tau_k), \\
&w_k = \exp(t_k)/\{\sum_{j=1}^{K} \exp(t_j)\}
\end{align*}
for $k=1, 2, \ldots, K$. 
We may then minimize $\mathbb{W}_N$ with respect to
$\{(\mu_k,\tau_k, t_k): k=1,2,\ldots,K\}$ over the unconstrained space $\mathbb{R}^{3K}$.
Furthermore, we adopt  the quasi-Newton BFGS algorithm
~\citep[Section 6.1]{nocedal2006numerical}.
To use this algorithm, it is best to provide the gradients of $\mathbb{W}_N(G)$,
which are given as follows:
\begin{align*}
&\frac{\partial}{\partial t_j} \mathbb{W}_N
=\sum_{k=1}^{K}\left\{\frac{\partial w_k}{\partial t_j}\frac{\partial}{\partial w_k}\mathbb{W}_N\right\}
=\sum_{k} w_j (\delta_{jk}-w_k)\frac{\partial}{\partial w_k}\mathbb{W}_N,  
\\
&\frac{\partial}{\partial \mu_j} \mathbb{W}_N
=2w_j\{\mu_j+\sigma_j\mu_0-\sum_{n=1}^Nx_{(n)}\Delta F_{nj}\},
\\
&\frac{\partial}{\partial \tau_j} \mathbb{W}_N
=2w_j\{\sigma_j(\mu_0^2+\sigma_0^2)
 +\mu_j\mu_0-\sum_{n=1}^Nx_{(n)}\Delta T_{nj}\}\frac{\partial \sigma_j}{\partial \tau_j}
\end{align*}
for $j=1,2,\ldots,K$ where
\begin{align*}
\frac{\partial}{\partial w_k} \mathbb{W}_N
=&\ 
 \{\mu_k^2 + \sigma_k^2(\mu_0^2+\sigma_0^2)
+2\mu_k\sigma_k\mu_0\} -2\sum_{n=1}^{N-1} \{x_{(n+1)}-x_{(n)}\}\xi_nF(\xi_n|\mu_k,\sigma_k)\\
&
- 2\big \{ \mu_k \sum_{n=1}^{N} x_{(n)} \Delta F_{nk}
+
\sigma_k \sum_{n=1}^{N} x_{(n)}\Delta T_{nk}
\big \}.
\end{align*}

Since $\mathbb{W}_N(G)$ is non-convex, the algorithm may find a local minimum of $\mathbb{W}_N(G)$ instead of a global minimum as required for MWDE.
We use multiple initial values for the BFGS algorithm, and regard the one 
with the lowest $\mathbb{W}_N(G)$ value as the solution.
We leave the algebraic details in the Appendix. 

This algorithm involves computing the quantiles $\xi_{n}$ and $\Delta T_{nj}$ repeatedly which may 
lead to high computational cost. Since $\xi_n\in [\min_{k} F^{-1}(n/N|\btheta_k), \max_{k} F^{-1}(n/N|\btheta_k)]$, it can be found efficiently via a bisection method. 
Fortunately, $T(x)$ has simple analytical forms
under two widely used location-scale mixtures which make
the computation of $\Delta T_{nj}$ efficient:
\begin{enumerate}
\item 
When $f_0(t)=(2\pi)^{-1/2}\exp(-x^2/2)$ which is the density function of the standard normal, 
we have  $t f_0(t)=-f_0'(t)$.
In this case, we find
$$ T(x) = -f_0(x).$$
\item
For finite mixture of location-scale logistic distributions, we have
\[
f_0(t)=\frac{\exp(-x)}{(1+\exp(-x))^2}
\]
and
\begin{equation}
\label{eq:Tx_logistic}
T(x) = \int_{-\infty}^{x} tf_0(t) dt=\frac{x}{1+\exp(-x)} -\log(1+\exp(x)).
\end{equation}
\end{enumerate}

\subsection{Penalized Maximum Likelihood Estimator}
A well investigated inference method under finite mixture of location-scale families 
is the pMLE \citep{tanaka2009strong,chen2008inference}. 
\citet{chen2008inference} consider this approach for finite normal mixture models. 
They recommend the following penalized log-likelihood function
\begin{equation*}
p\ell_N(G|\mathcal{X}) 
= 
\ell_N(G|\mathcal{X}) - a_N\sum_{k} 
\left \{s_x^2/\sigma_k^2 + \log \sigma_k^2 \right \}
\end{equation*}
for some positive $a_N$ and sample variance $s_{x}^2$. 
The log-likelihood function is given in~\eqref{eq:log_likelihood}.
They suggest to learn the mixing distribution $G$ via pMLE defined as
\begin{equation*}
\label{eq:pMLE}
    \hat{G}_{N}^{\text{pMLE}} = \argsup p\ell_N(G|\mathcal{X}).
\end{equation*}
The size of $a_N$ controls the strength of the penalty and a recommended value is $N^{-1/2}$. 
Regularizing the likelihood function via a penalty function fixes the problem 
caused by degenerated subpopulations (i.e. some $\sigma_k=0$).
The pMLE is shown to be strongly consistent when the number of components has a known upper bound under the finite normal mixture model. 

The penalized likelihood approach can be easily extended to finite mixture of location-scale families. 
Let $f_0(\cdot)$ be the density function in the location-scale family as before. 
We may replace the sample variance $s_x^2$ in the penalty function by any 
scale-invariance statistic such as the sample inter-quartile range.
This is applicable even if the variance of $f_0(\cdot)$ is not finite.

We can use the EM algorithm for numerical computation.
Let $\mathbf{z}_n = (z_{n1}, \ldots, z_{nK})$ be the membership vector of the $n$th observation. 
That is, the $k$th entry of $\mathbf{z}_n$ is 1 when the response value $x_n$ is 
an observation from the $k$th subpopulation and 0 otherwise. 
When the complete data $\{(\mathbf{z}_n, x_n), n=1,2,\ldots,N\}$ are available, 
the penalized complete data likelihood function of $G$ is given by
\[
    p\ell_{N}^c(|\mathcal{X}) 
    = \sum_{n=1}^{N} \sum_{k=1}^K z_{nk}
    \log \left\{\frac{w_k}{\sigma_k} f_0\Big (\frac{x_i-\mu_k}{\sigma_k}\Big ) \right\}
    -a_N \sum_{k} \big  \{ s_x^2/\sigma_k^2 + \log (\sigma_k^2) \big \}.
\]
Given the observed data $\Xcal$ and proposed mixing distribution $G^{(t)}$, 
we have the conditional expectation
\[
w_{nk}^{(t)} 
= \mathbb{E}(z_{nk}| \Xcal, G^{(t)})
= \frac{w_k^{(t)} f(x_n|\mu_k^{(t)}, \sigma_k^{(t)})}
           {\sum_{j=1}^{K} w_j^{(t)} f(x_n| \mu_j^{(t)},\sigma_j^{(t)})}.
\]
After this E-step, we define
\begin{align*}
    Q(G|G^{(t)}) 
     =& 
    \sum_{n=1}^{N} \sum_{k=1}^K w_{nk}^{(t)}\log \left\{\frac{w_k}{\sigma_k}f_0
  \left(\frac{x_n-\mu_k}{\sigma_k}\right) \right\}
    -a_N \sum_{k} \big  \{s_x^2/\sigma_k^2 + \log (\sigma_k^2) \big \}.
\end{align*}
Note that the subpopulation parameters are well separated in $Q(\cdot|\cdot)$.
The M-step is to maximize $Q(G|G^{(t)}) $ with respect to $G$. 
The solution is given by the mixing distribution $G^{(t+1)}$ with mixing weights
$$
w_{k}^{(t+1)} =   N^{-1} \sum_{n=1}^{N} w_{nk}^{(t)}
$$
and the subpopulation parameters 
\begin{equation}
\label{eq:m_step}
\btheta_{k}^{(t+1)}
= \argmin_{\theta}
\Big \{
 \sum_{n} w_{nk}^{(t)} 
\{\log \sigma - f(x_n|\btheta)\} + a_N\{s_x^2/\sigma^2 + \log\sigma^2\}
\Big \}
\end{equation}
with the notational convention $\btheta = (\mu, \sigma)$.

For general location-scale mixture, the M-step \eqref{eq:m_step} may not have a closed form solution
but it is merely a simple two-variable function.
There are many effective algorithms in the literature to solve this optimization problem.
The EM-algorithm for pMLE increases the value of the penalized likelihood after each iteration. 
Hence, it should converge as long as the penalized likelihood function has an upper bound.
We do not give a proof as it is a standard problem.

\section{Experiments}
\label{sec:exp}
We now study the performance of MWDE and pMLE under finite location-scale mixtures.
We explore the potential advantages of the MWDE and quantify its efficiency loss, 
if any, by simulation experiments. 
Consider the following three location-scale families~\citep{chen2020homogeneity}:
\begin{enumerate}
\item 
Normal distribution: $f_0(x)=(2\pi)^{-1/2}\exp(-x^2/2)$. 
Its mean and variance are given by $\mu_0=0$ and $\sigma_0^2 = 1$.
\item 
Logistic distribution: $f_0(x) = {\exp(-x)}/{(1+\exp(-x))^2}$.
Its mean and variance are given by $\mu_0=0$ and  $\sigma_0^2 = {\pi^2}/{3}$.
\item 
Gumbel distribution (type I extreme-value distribution): $f_0(x)= \exp(-x-\exp(-x))$.
Its mean and variance are given by 
$\mu_0=\gamma$ and $\sigma_0^2 = {\pi^2}/{6}$
where $\gamma$ is the Euler constant.
\end{enumerate}

We will also include a real data example to compare the image segmentation result of using the MWDE and pMLE.

\subsection{Performance Measure}

For vector valued parameters, the commonly used performance metric of their estimators is 
the mean squared error (MSE). 
A mixing distribution with finite and fixed support points can be regarded as 
a real-valued vector in theory.
Yet the mean squared errors of the mixing weights, the subpopulation means, 
and the subpopulation scales are not comparable in terms of the learned finite mixture.
In this study, we use two performance metrics specific for finite mixture models. 
Let $\hat G$ and $G^*$ be the learned mixing distribution and the true mixing distribution.
We use $L_2$ distance between the learned mixture and the true mixture
as the first performance metric. 
The $L_2$ distance between two mixtures $f(\cdot|G)$ and $f(\cdot|\tilde G)$ is defined to be
\begin{equation*}
L_2(f(\cdot|G),f(\cdot|\tilde G)) 
= \{
\bw^{\tau} S_{GG} \bw - 2\bw^{\tau} S_{G\tilde G}  \tilde{\bw} 
+ \tilde{\bw}^{\tau} S_{\tilde G \tilde G} \tilde{\bw} \}^{1/2}
\end{equation*}
where $S_{GG}, S_{G\tilde G}$ and $S_{\tilde G \tilde G}$
are three square matrices of size $K \times K$ with their $(n,m)$th
elements given by
\[
\int f(x|\btheta_n)f(x|\btheta_m) dx,
\quad\int f(x|\btheta_n)f(x|\tilde\btheta_m) dx ,
\quad\int f(x|\tilde\btheta_n)f(x|\tilde\btheta_m) dx.
\]

Given an observed value $x$ of a unit from the true mixture population, 
by Bayes' theorem, the most probably membership of this unit is given by
\[
 k^*(x) = \argmax_{k} \{w_k^* f^*(x | \btheta_k^*)\}.
\]
Following the same rule, if $\hat G$ is the learned mixing distribution,
then the most likely membership of the unit with observed value $x$ is
\[
\hat k(x) = \argmax_{k} \{\hat w_k f(x|\hat \btheta_k)\}.
\]
We cannot directly compare $k^*(x)$ and $\hat k(x)$
because the subpopulation themselves are not labeled. 
Instead, the adjusted rand index (ARI) is a good performance metric
for clustering accuracy. 
Suppose the observations in a dataset are divided into $K$ clusters
$A_1, A_2, \ldots, A_K$ by one approach,
and 
$K'$ clusters $B_1, B_2, \ldots, B_{K'}$ by another. 
Let $N_i = \# (A_i), ~ M_j = \# (B_j), ~~ N_{ij} = \# (A_i B_j)$ for 
$i, j=1, 2, \ldots, K$, where $\#(A)$ is the number of units in set $A$.
The ARI between these two clustering outcomes is defined to be
\[
\text{ARI}= \dfrac{
\sum_{i, j } \binom{N_{ij}}{2}
-  \binom{N}{2}^{-1} \sum_{i, j}\binom{N_{i}}{2}\binom{M_{j}}{2}}
{\frac{1}{2} \sum_{i}\binom{N_{i}}{2} + \frac{1}{2}  \sum_{j} \binom{M_{j}}{2}
- \binom{N}{2}^{-1} \sum_{i, j}\binom{N_{i}}{2}\binom{M_{j}}{2}}.
\]
When the two clustering approaches completely agree with each other, 
the ARI value is $1$. When data are assigned to clusters randomly, the 
expected ARI value is $0$.
ARI values close to 1 indicate a high degree of agreement. 
We compute ARI based on clusters formed by $k^*(x)$ and $\hat k(x)$.

For each simulation, we choose or generate a mixing distribution $G^{*(r)}$, then generate a random sample from mixture $f (x | G^{*(r)})$.
This is repeated $R$ times.
Let $\hat G^{(r)}$ be the learned $G$ based on the $r$th data set.
We obtain the two performance metrics as follows:

\begin{enumerate}
\item
Mean $L_2$ distance:
$$
\text{ML2} = R^{-1} \sum_{r=1}^R L_2( f (\cdot|\hat{G}^{(r)}), f(\cdot|G^{*(r)})).
$$

\item 
Mean adjusted rand index:
$$
\text{MARI} = R^{-1} \sum_{r=1}^R \mbox{ARI}(\hat{G}^{(r)}, G^{*(r)} ).
$$
\end{enumerate}
The lower the ML2 and the higher the MARI, the better the estimator
performs.

\subsection{Performance under Homogeneous Model}

The homogeneous location-scale model is a special mixture model 
with a single subpopulation $K=1$. 
Both MWDE and MLE are applicable for parameter estimation. 
There have been no studies of MWDE in this special case in the literature.
It is therefore of interest to see how MWDE performs under this model.

Under three location-scale models given earlier, the MWDE has closed analytical forms.
Using the same notation introduced, their analytical forms are as follows.

\begin{enumerate}
  \item 
  Normal distribution:
\[
   \hat{\mu}^{\text{MWDE}} 
   = \bar{x},
   ~\hat{\sigma}^{\text{MWDE}} = \sum_{n=1}^N x_{(n)} \left\{f_0(\xi_{n-1})-f_0(\xi_n)\right\}.
\]

\item 
Logistic distribution:
\begin{equation*}
    \hat{\mu}^{\text{MWDE}}   = \bar{x},
    ~\hat{\sigma}^{\text{MWDE}} 
    = \frac{3}{\pi^2}\sum_{n=1}^N x_{(n)} \left\{T(\xi_n)-T(\xi_{n-1})\right\}
\end{equation*}
where $T(x)$ is given in~\eqref{eq:Tx_logistic}.

\item 
Gumbel distribution:
\begin{equation*}
    \hat{\mu}^{\text{MWDE}} = \{1-\gamma r\}^{-1}\{\bar{x}-\gamma T\},
     ~\hat{\sigma}^{\text{MWDE}} = T - r \hat{\mu}^{\text{MWDE}}
\end{equation*} 
where 
\[
T=\{\gamma^2 + \pi^2/6\}^{-1}\sum_{n=1}^N x_{(n)} \int_{\xi_{n-1}}^{\xi_{n}} tf_0(t)dt
\]
and $r=\gamma/(\gamma^2+\pi^2/6)$.
\end{enumerate}

The MLEs under the logistic and Gumbel distributions do not have an easy to use analytical form.
We  employ a numerical optimization program to solve for MLE.
We generate samples of sizes between $N=10$ to $N=100000$ with $R = 1000$ repetitions. 
Under the homogeneous model, it is most convenient to compute the MSE 
of the location and scale parameters separately.
Due to invariance property, we generate data from distributions
with $\mu=0$ and $\sigma = 1$.
The simulation results are summarized as plots in Figure~\ref{fig:homogeneous_MWDEvsMLE}.
\begin{figure}[htpb]
 \centering      
\subfloat[Normal]{\includegraphics[width=0.3\textwidth]{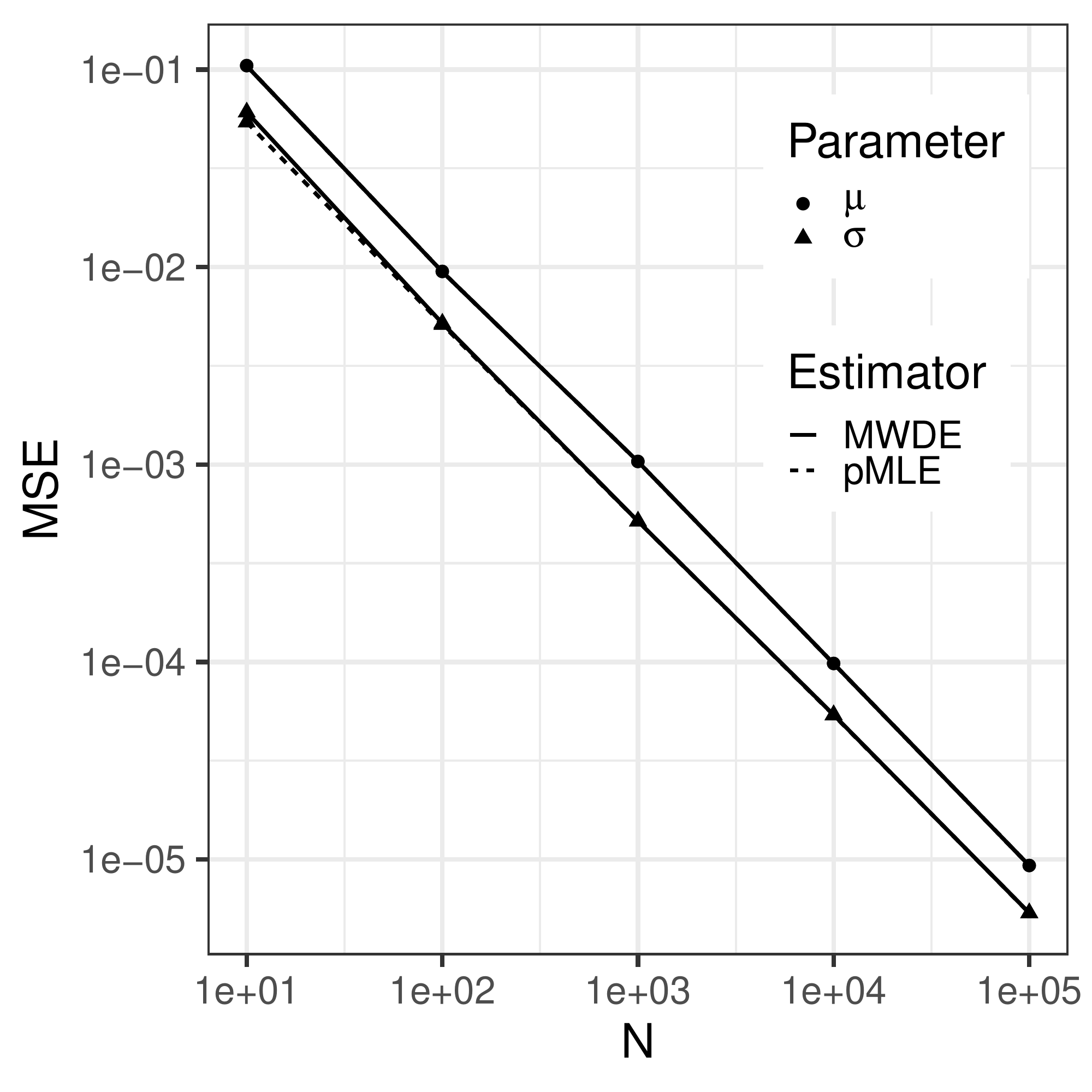}}
\subfloat[Logistic]{\includegraphics[width=0.3\textwidth]{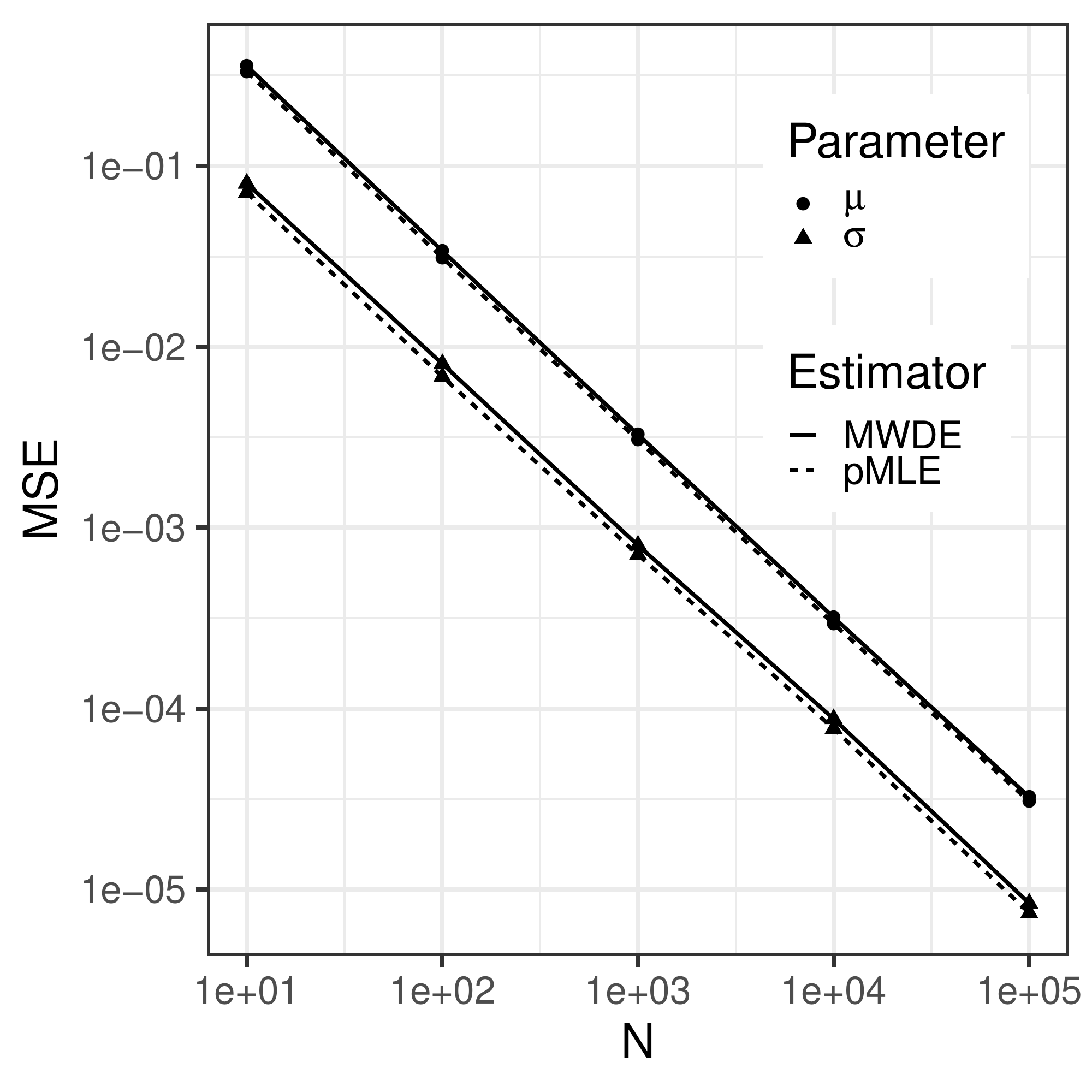}}
\subfloat[Gumbel]{\includegraphics[width=0.3\textwidth]{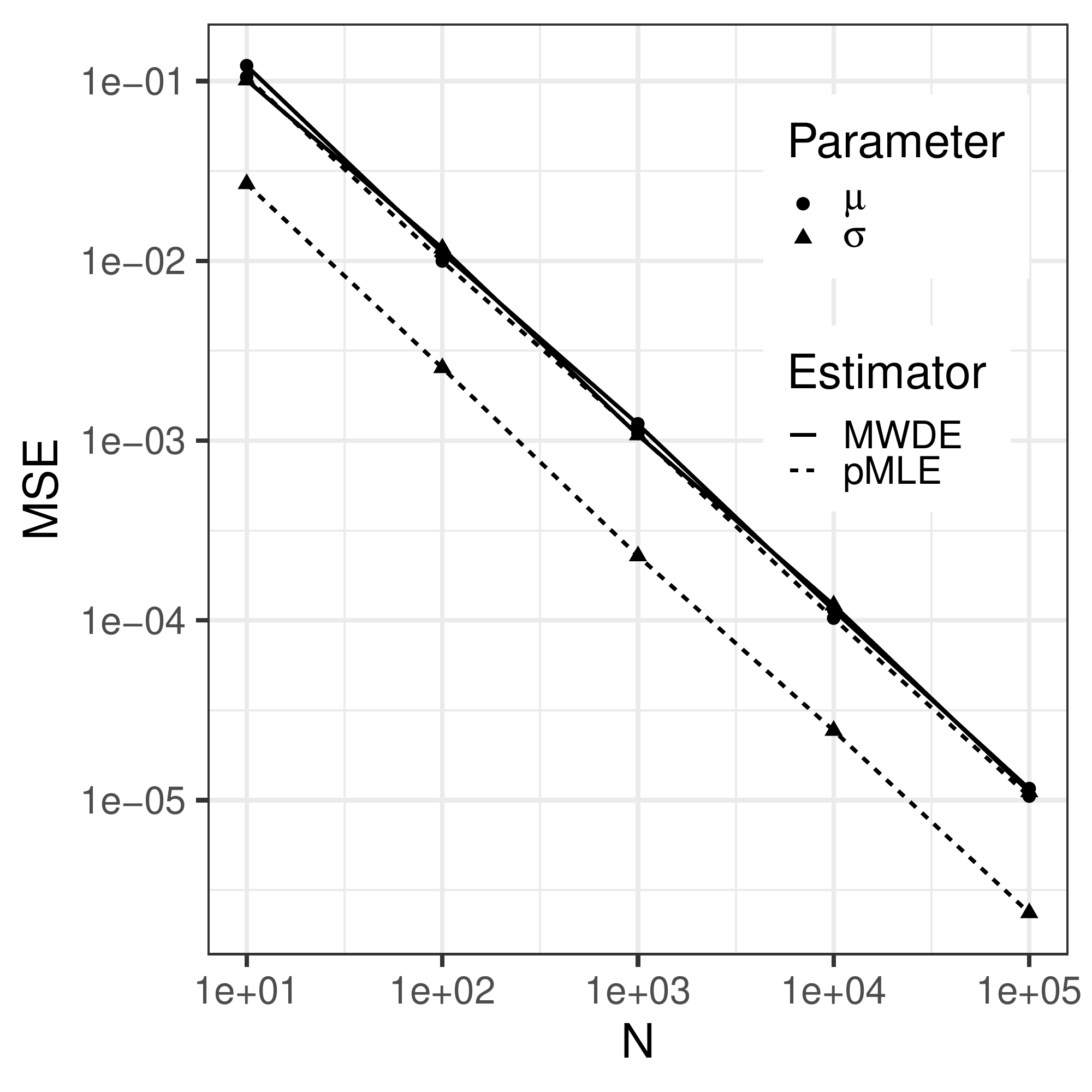}}
\caption{The MSEs of the MWDE and MLE for location and scale parameters 
versus sample size $N$ under different homogeneous models.}
\label{fig:homogeneous_MWDEvsMLE}
\end{figure}
Both the x and y axes in these plots are in logarithm scale.
For both MLE and MWDE, their log-MSE and $\log (N)$ values are
close to the straight lines with slope $- 1$. 
This phenomenon indicates that both estimators have the expected convergence rates $O(N^{-1/2})$ as the sample size $N \to \infty$.

The performance of the estimators for the location parameter and scale parameter are different.
For the location parameter under all three models, the lines formed by MLE and MWDE are nearly indistinguishable though the MLE is always below the MWDE. 
For the scale parameter $\sigma$, the MLE is also more efficient than the MWDE but the difference is negligible under the normal and logistic models.
Under the Gumbel model, the MWDE is less efficient.

In summary, using MWDE under a homogeneous model may not be preferred but may be acceptable under the normal and logistic models.
We do not investigate the performance of MWDE under Gumbel mixture due to its efficiency loss under the homogeneous model.
With these observations, we move to its performance under finite location-scale mixtures.

\subsection{Efficiency and Robustness under Finite Location-Scale Mixtures}
We next study the efficiency and robustness of the MWDE for learning finite location-scale mixtures.
Since the MLE is not well-defined, we compare the performance of MWDE with the pMLE~\citep{chen2009inference} instead. 
We compare their performances when the mixture model is correctly specified,
when the data is contaminated, or when the model is mildly misspecified.

\subsubsection{Efficiency}
A widely employed two-component mixture model~\citep{cutler1996minimum,zhu2016two} 
has a density function in the following from:
\be
\label{eq:two-component}
f(x | G) = p f(x | 0, a)+(1-p)f(x | b, 1)
\ee
for some density function $f(\cdot | \btheta)$ from a location-scale family.
Namely, we have $K=2$ is known, the mixing weights be $w_1 = p, w_2 = 1-p$, 
and subpopulation parameters be $\btheta_1 = (0, a)$ and $\btheta_2 = (b, 1)$.
By choosing different combinations of $p$, $a$, and $b$, we obtain 
mixtures with different properties.
Due to the invariance property, we need only consider the case where one of 
the location parameters is $0$, and one of the scale parameter is $1$. 

We generate samples from $f(x| G)$ according to the following scheme: 
generate an observation $Y$ from distribution with density function $f_0 (x)$ and let
\begin{equation}
\label{eq:data_generation}
X = \begin{cases}
    a Y, & \text{with probability } p; \\
    Y + b, & \text{otherwise.}
\end{cases}
\end{equation}
We can easily see that the distribution of $Y$ is $f(x|G)$ specified earlier.

The level of difficulty to precisely estimate the mixture largely depends on the degree of overlap between the subpopulations. 
Let 
\[
o_{j|i} 
= \mathbb{P}\big ( 
w_{i}f(X|\mu_{i},\sigma_{i}) < w_{j}f(X|\mu_{j},\sigma_{j}) 
 |X \sim f(x|\mu_{i},\sigma_{i})
 \big ).
\]
This is the probability of a unit from subpopulation $i$ misclassified as a unit in
subpopulation $j$ by the maximum posterior rule. 
The degree of overlap between the $i$th and $j$th subpopulations is therefore 
\begin{equation}
\label{eq:pairwise_overlap}
o_{ij} = o_{j|i} + o_{i|j}.
\end{equation}

We employ the following $a$, $b$, and $p$ values in our simulation experiments:
\begin{enumerate}
\item 
mixing proportion $p=0.15, 0.25, 0.5, 0.75, 0.85$;
\item 
scale of the first subpopulation $a^2=1, 2$;
\item 
location parameter $b$ values such that $o_{12} = 0.03, 0.1$. 
\end{enumerate}
The combination of these choices leads to $24$ mixtures with various shapes.
The sample size $N$ in our experiments
is chosen to be $100$, $500$, and $1000$ respectively.

We obtain the average $L_2$ distance (ML2) and adjusted rand index (MARI) based on  
$R=1000$ repetitions on data generated from normal and logistic mixture distributions 
as specified by \eqref{eq:data_generation}.
Figures~\ref{fig:correct_normal_efficiency} and 
 \ref{fig:correct_logistic_efficiency}, respectively, contains
plots of ML2 and MARI of the WMDE and pMLE estimators 
against sample size $N$ under these two models. 
\begin{figure}[htbp]
\centering      
\subfloat[$L_2$ distance]
{\includegraphics[width=\textwidth]
{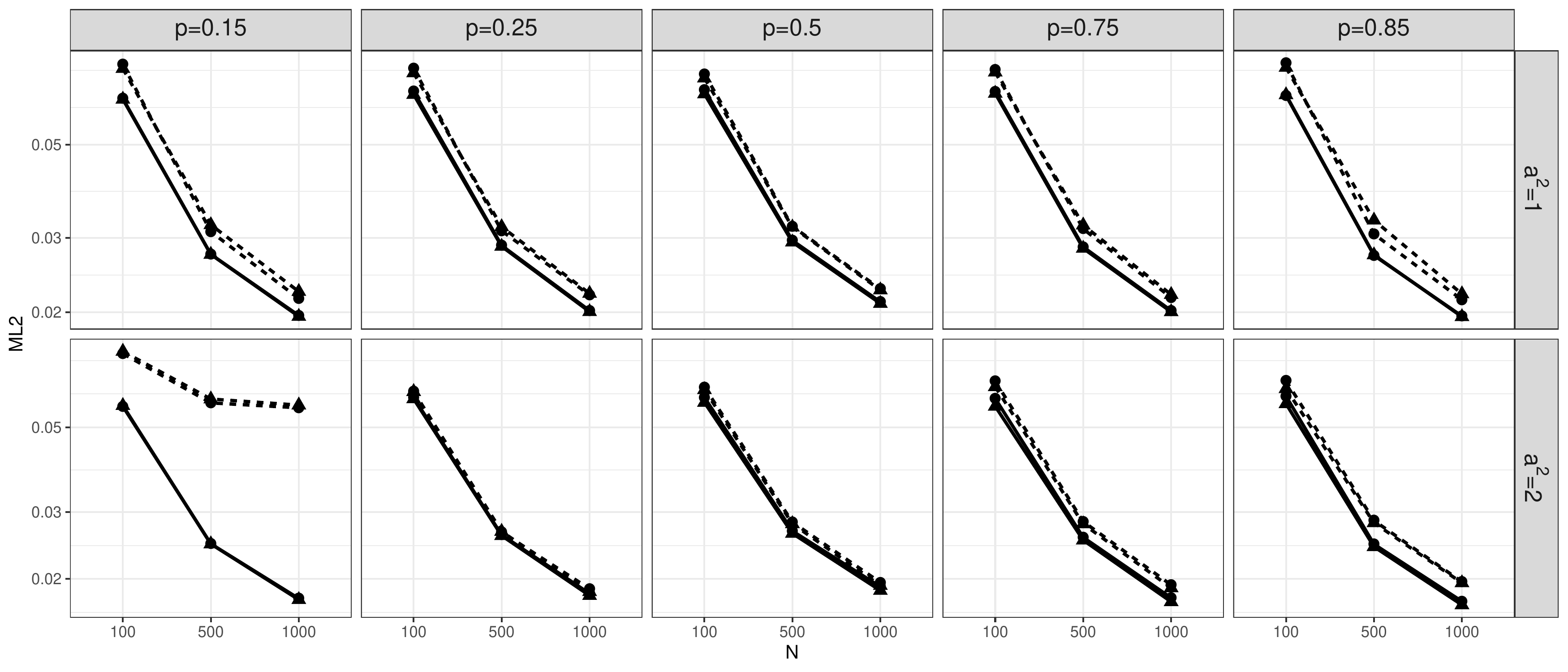}}\\
\subfloat[Adjusted rand index]
{\includegraphics[width=\textwidth]{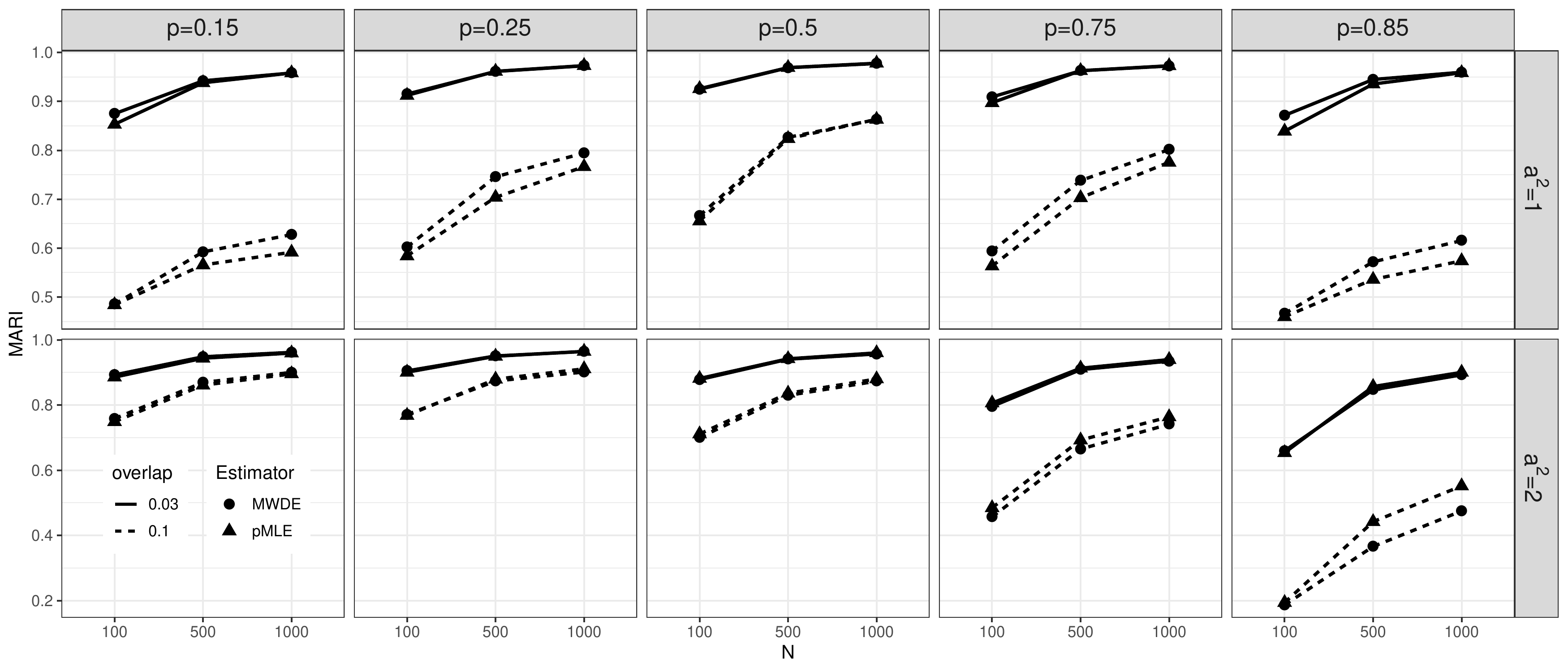}}
\caption{
Performances of pMLE and MWDE under 2-component normal mixture.
}
\label{fig:correct_normal_efficiency}
\end{figure}
\begin{figure}[htbp]
\centering      
\subfloat[$L_2$ distance]{\includegraphics[width=\textwidth]
{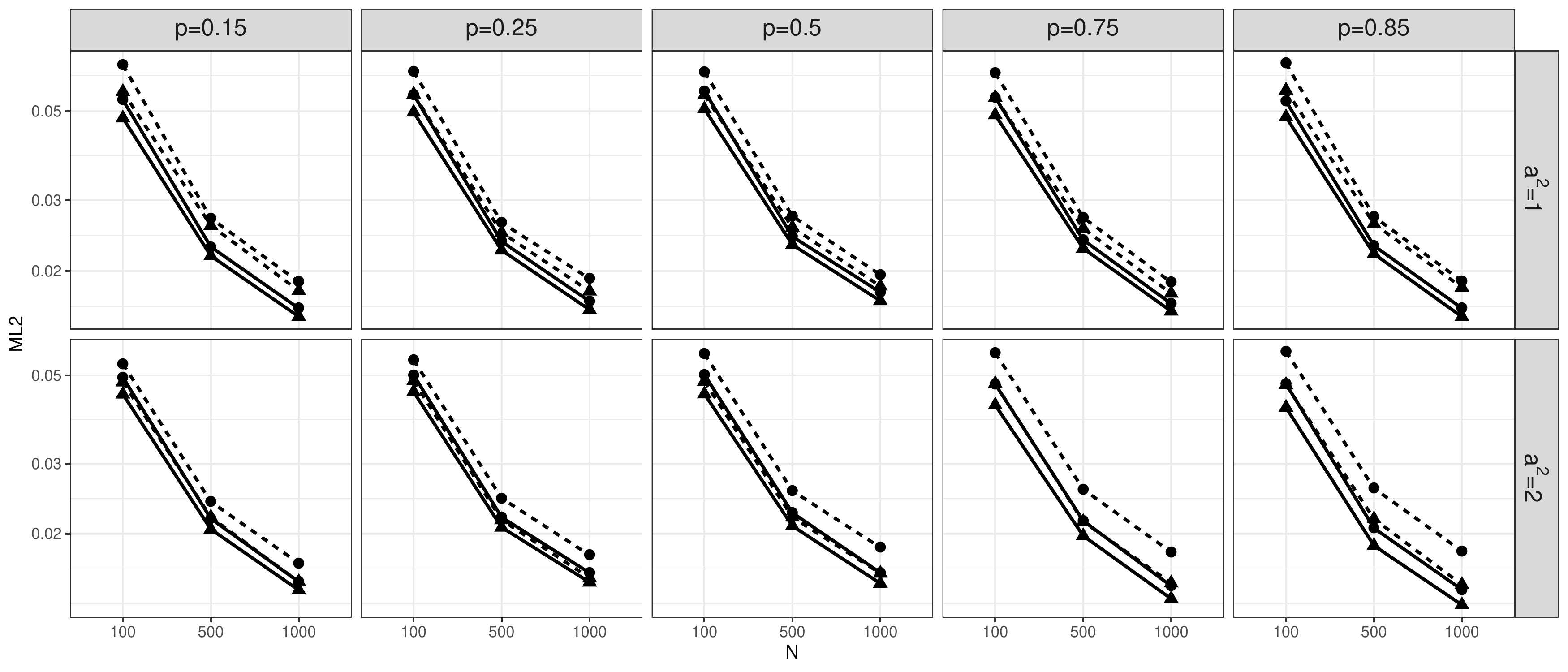}}\\
\subfloat[Adjusted rand index]{\includegraphics[width=\textwidth]
{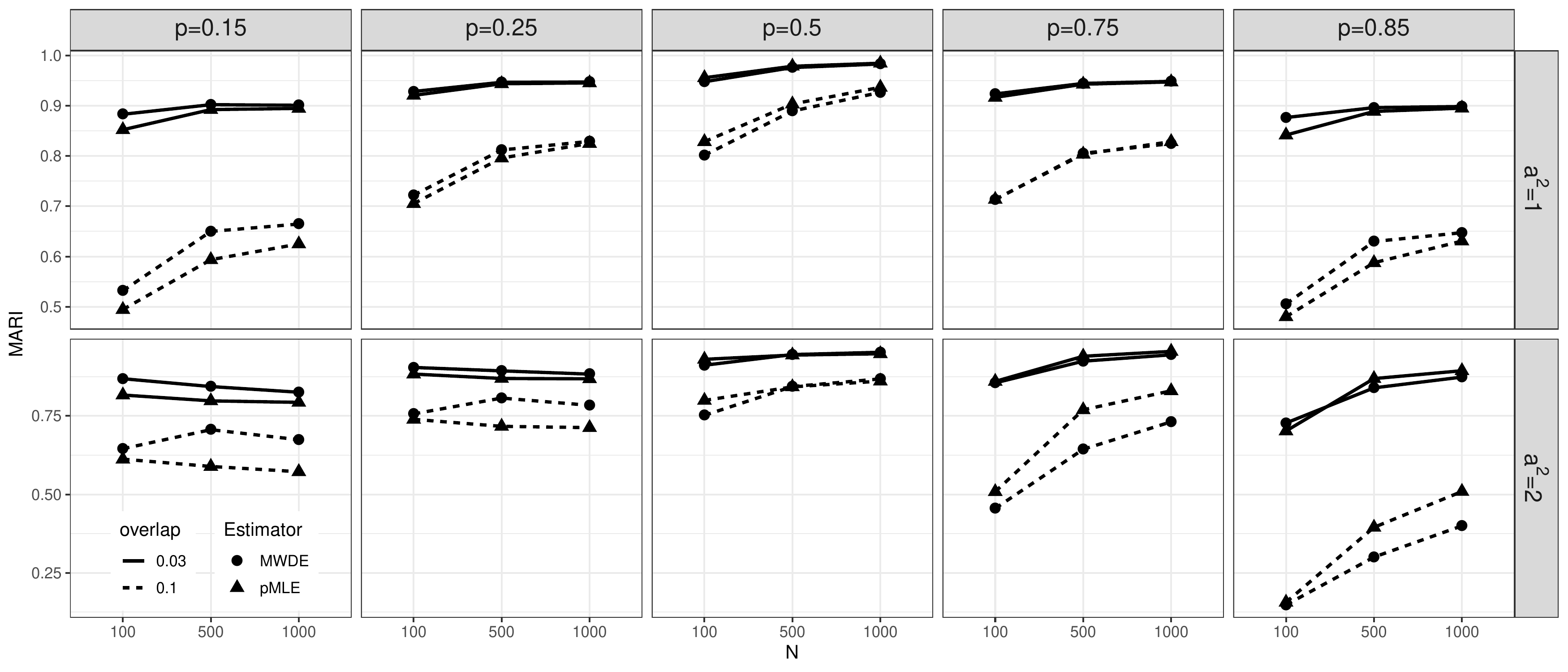}}
\caption{Performances of pMLE and MWDE under 2-component logistic mixture.}
\label{fig:correct_logistic_efficiency}
\end{figure}
We can see that when the sample size increases, ML2 of both estimators decrease 
and MARI of both estimators increase, 
supporting the theory that both WMDE and pMLE are consistent.
Under the normal mixture, these two estimators have nearly equal $L_2$ distances. 
The MWDE slightly outperforms pMLE in terms of the MARI, 
when the degree of overlap is large ($o_{12} = 0.1$) and the two subpopulations 
have both equal scale and highly unbalanced weights.
Under logistic mixture, as shown in plots (a) and (b) of
Figures~\ref{fig:correct_logistic_efficiency}, 
the pMLE always outperforms the MWDE in terms of the $L_2$ distance. 
In terms of the MARI, the MWDE is better when the scale parameters 
are equal and weights are highly unbalanced. 
When the scale parameters are different, the pMLE is better than MWDE when $p>0.5$ and worse than MWDE when $p<0.5$.

We next investigate the performance of the MWDE and pMLE
for learning 3-component normal mixtures.
We come up with 8 such distributions with different configurations.
The three subpopulations have the same or different weights and same or different scale parameter values. 
They lead to different degrees of overlap as defined by
\[
\texttt{MeanOmega} = \text{mean}_{1 \leq i < j \leq 3} \{ o_{ij} \}.
\]
where $o_{ij}$ is the degree of overlap between subpopulations
$i$ and $j$ in~\eqref{eq:pairwise_overlap}.
See Table \ref{tab:mixture_param} for detailed parameter values.

\begin{table}[htpb]
\centering
\small
\caption{Parameter values of 3-component normal mixtures.}
\begin{tabular}{lcccccccccc}
\toprule
&\texttt{MeanOmega} 
& $w_1$ & $w_2$ & $w_3$ & $\mu_1$ & $\mu_2$ & $\mu_3$ & $\sigma_1$ & $\sigma_2$ & $\sigma_3$ \\
\midrule
I&0.288 (low)&0.4&0.5&0.1&-2&0&1&0.3&2&0.4\\
II&0.367 (high)&0.4&0.5&0.1&-2&0&1&0.3&1&0.4\\    
\midrule
III&0.097 (low)&0.3&0.5&0.2&-3&0&3&1&1&1\\
IV&0.249 (high)&0.3&0.5&0.2&-2&0&2&1&1&1\\    
\midrule
V&0.148 (low)&1/3&1/3&1/3&-1&0&1&1.5&0.1&0.5\\
VI&0.267 (high)&1/3&1/3&1/3&-0.5&0&0.5&1.5&0.1&0.5\\
\midrule
VII&0.091 (low)&1/3&1/3&1/3&-3&0&3&1&1&1\\
VIII&0.226 (high)&1/3&1/3&1/3&-2&0&2&1&1&1\\
\bottomrule
\end{tabular}
\label{tab:mixture_param}
\end{table}

Figure~\ref{fig:correct_3comp} contains plots of the ML2 and MARI values of two estimators.
It is seen that the pMLE consistently outperforms MWDE in terms of ML2 but the difference is small.
The performances of the MWDE and pMLE are mixed in terms of MARI and the differences are small.
The pMLE is clearly better under the I and II.

\begin{figure}[!htpb]
\centering      
\subfloat[$L_2$ distance]{\includegraphics[width=0.45\textwidth]
{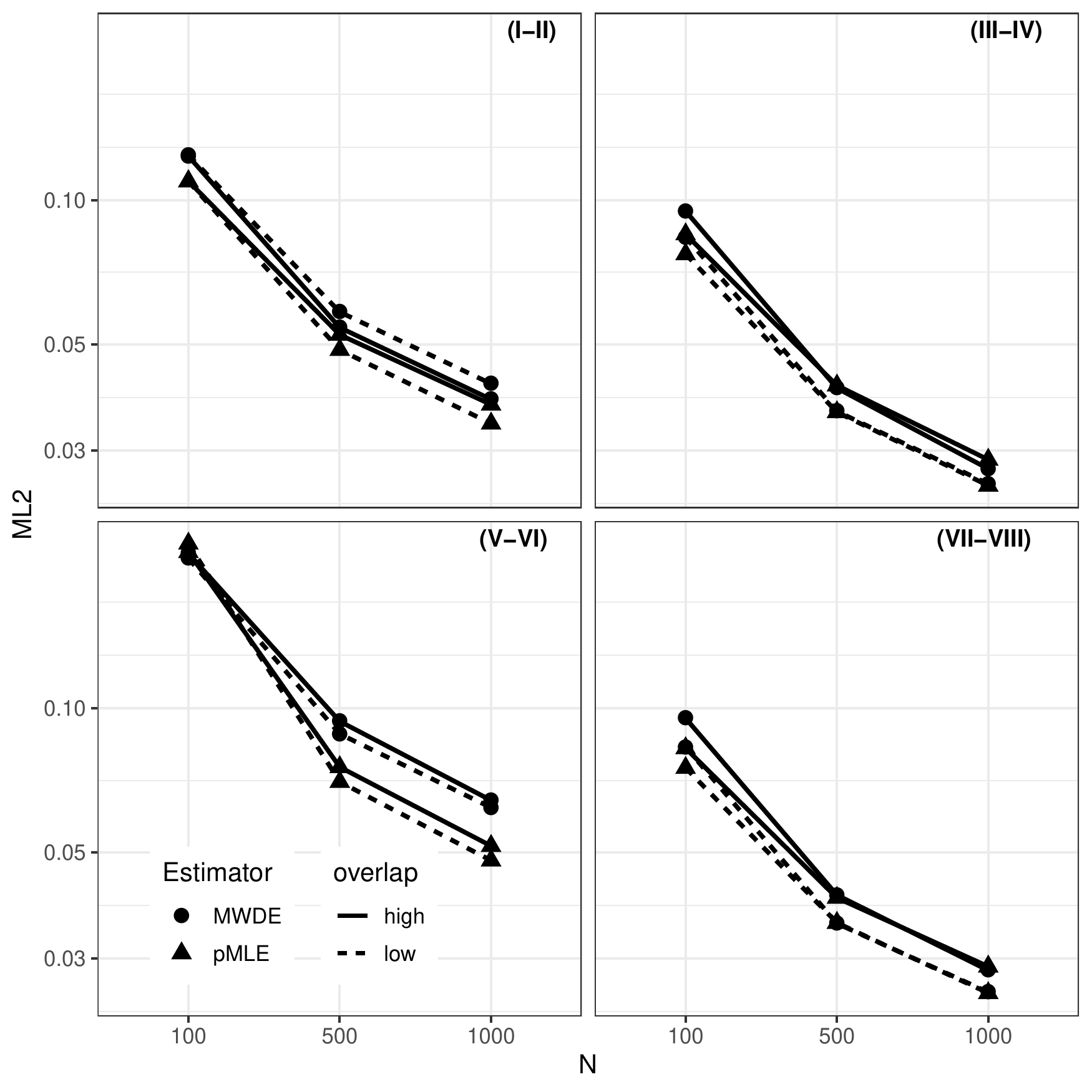}}
\subfloat[Adjusted rand index]
{\includegraphics[width=0.45\textwidth]{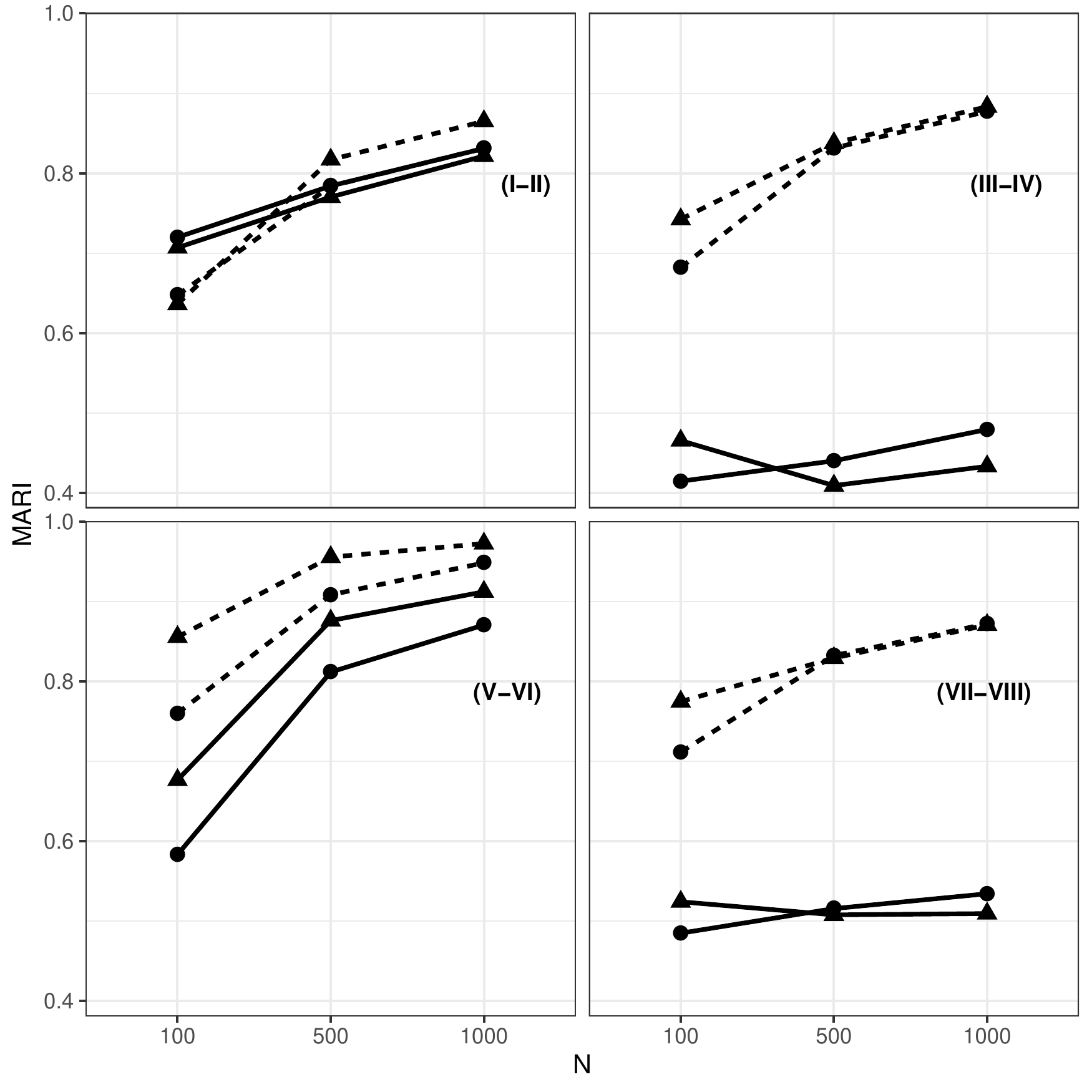}}
\caption{Performances of pMLE and MWDE under 3-component normal mixture.}
\label{fig:correct_3comp}
\end{figure}

\subsubsection{Robustness}
Robustness is another important property of estimators.
Sample mean is the most efficient unbiased
estimator of the population mean in terms of variance under normality
or some other well known parametric models. 
However, the value of the sample mean changes dramatically 
even if the data set contains merely a single extreme value.
Sample median offers a respectable alternative and still has high efficiency
across a broader range of parametric models.

In the context of learning finite location-scale mixture models, 
both pMLE and MWDE rely on a parametric distribution family
assumption through $f_0(x)$. 
How important is to have $f_0(x)$ correctly specified?
We shed some light into this problem by simulation experiments in this section.
We learn finite normal mixtures assuming $K=2$ but
generate data from the following distributions:
\begin{enumerate}
 \item
Mixture with outliers:
$ (1-\alpha)\{ p \phi (x | 0, a)+(1-p) \phi (x | b, 1)\}+ \alpha \phi(x; 8, 1)$ 
with $\alpha=0.01$ and $\phi(x|\mu,\sigma) = 1/\sqrt{2\pi\sigma^2}\exp(-(x-\mu)^2/\sigma^2)$.

\item
Mixture contaminated:
$ (1-\alpha)\{ p \phi (x | 0, a)+(1-p) \phi (x | b, 1)\} + \alpha \phi(x; b/2, 7)$ with $\alpha=0.01$.

\item
Mixture mis-specified I:
$p f_0(x |0, a)+(1-p)f_0 (x |b, 1)$ with $f_0 (x)$ being Student-t with $4$ degrees of freedom.

\item
Mixture mis-specified II:
$p f_1(x |0, a)+(1-p) f_2(x | b, 1)$ with 
$f_1(x)$and $f_2(x)$ being Student-t with $2$ and $4$ degrees of freedom.
\end{enumerate}

In every case, we use the combinations of the $a$, $b$, and $p$ value-combinations the same as before.
We regard $ (1-\alpha)\{ p \phi (x | 0, a)+(1-p) \phi (x | b, 1)\}$
as the true distribution in all cases and computed the MARI accordingly.

\begin{figure}[htpb]
\centering  
\subfloat[Mixture with outliers]
{\includegraphics[width=0.9\textwidth]
{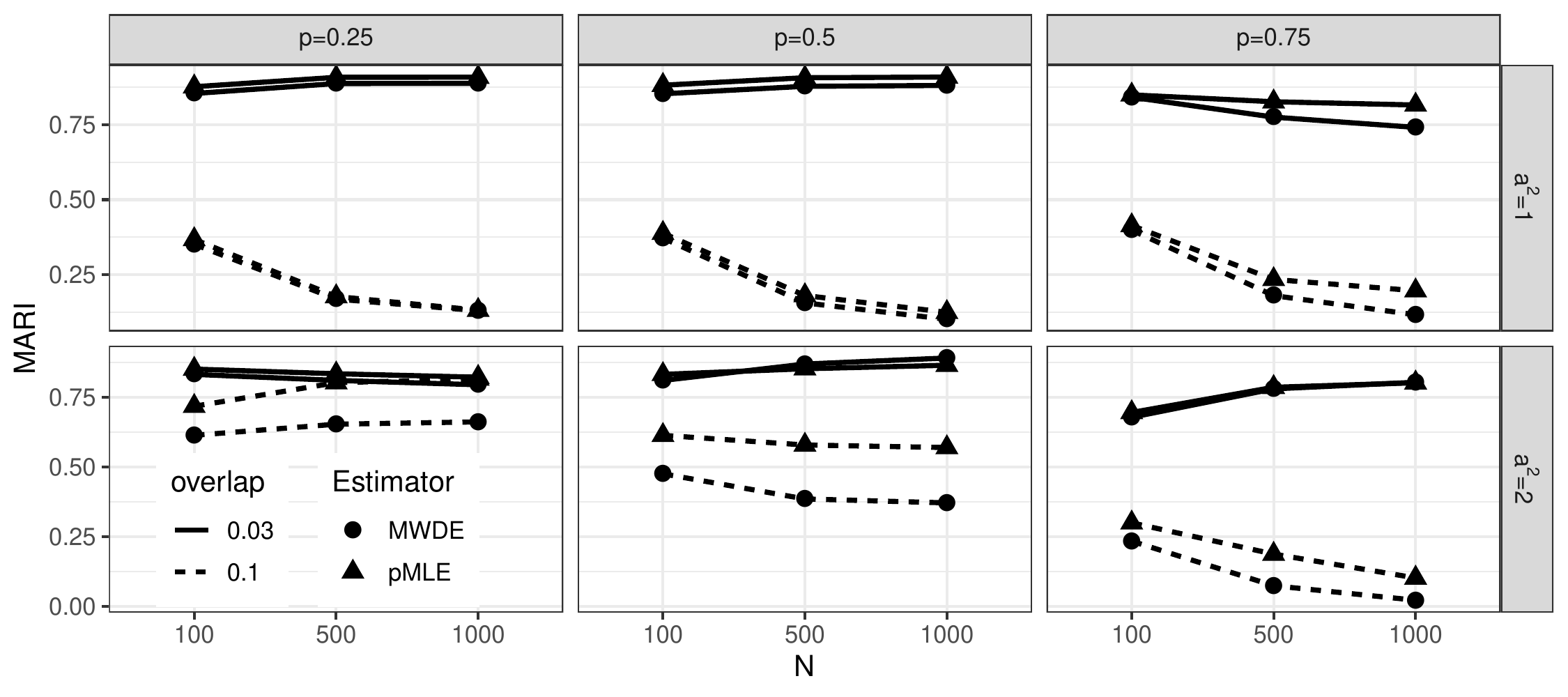}}\\
\subfloat[Mixture contaminated]
{\includegraphics[width=0.9\textwidth]
{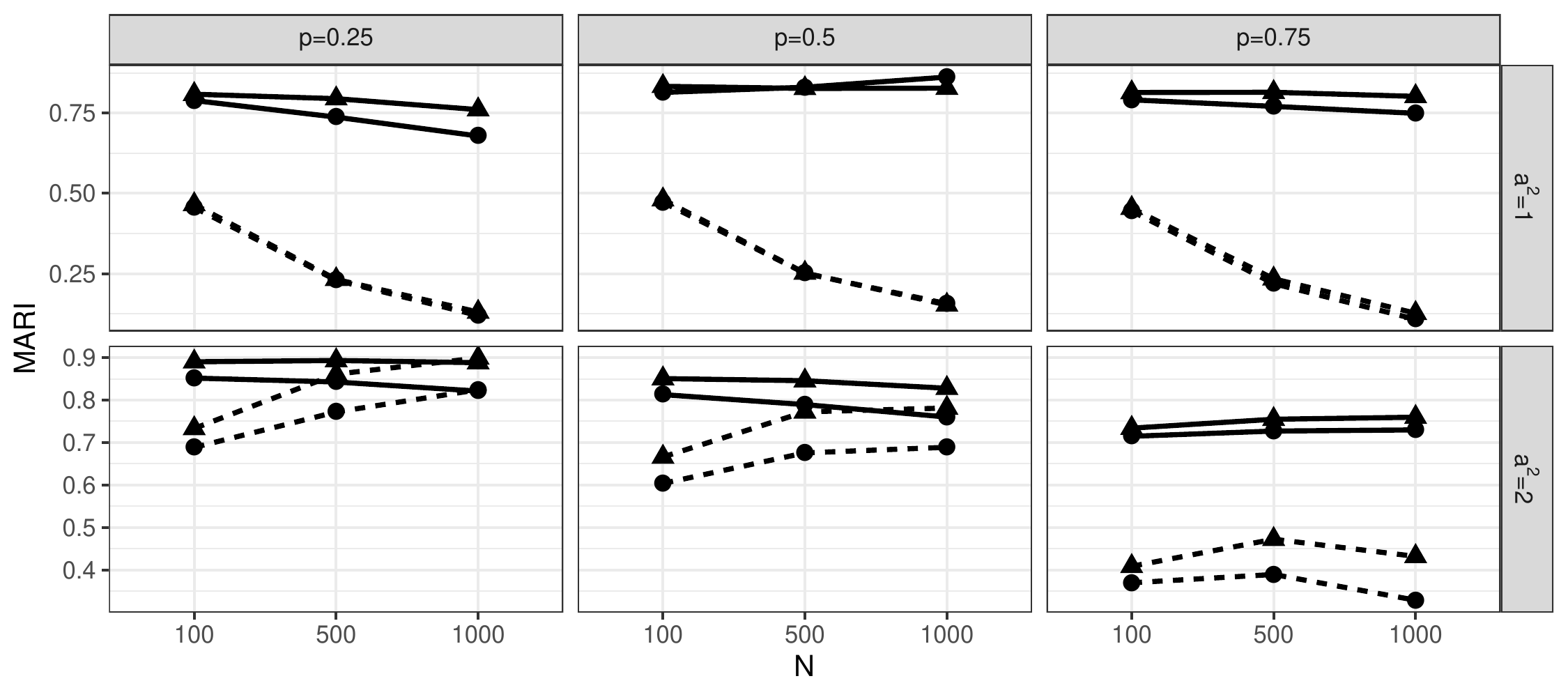}}
\caption{Adjusted rand index based on pMLE and MWDE
when data contains outliers or is contaminated}
\label{fig:robust1}
\end{figure}

\begin{figure}[hptb]
\centering  
\subfloat[Mixture mis-specified I]
{\includegraphics[width=0.9\textwidth]
{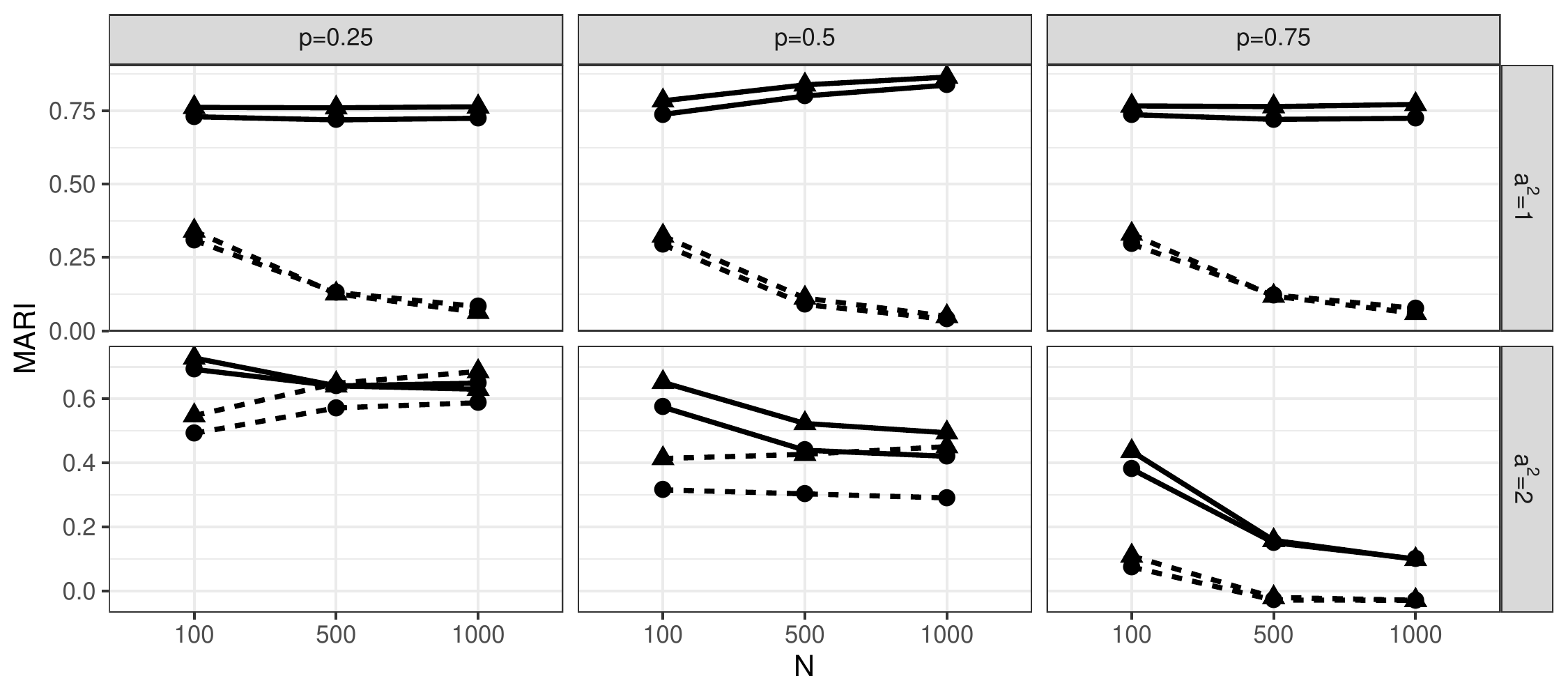}}\\
\subfloat[Mixture mis-specified II]
{\includegraphics[width=0.9\textwidth]
{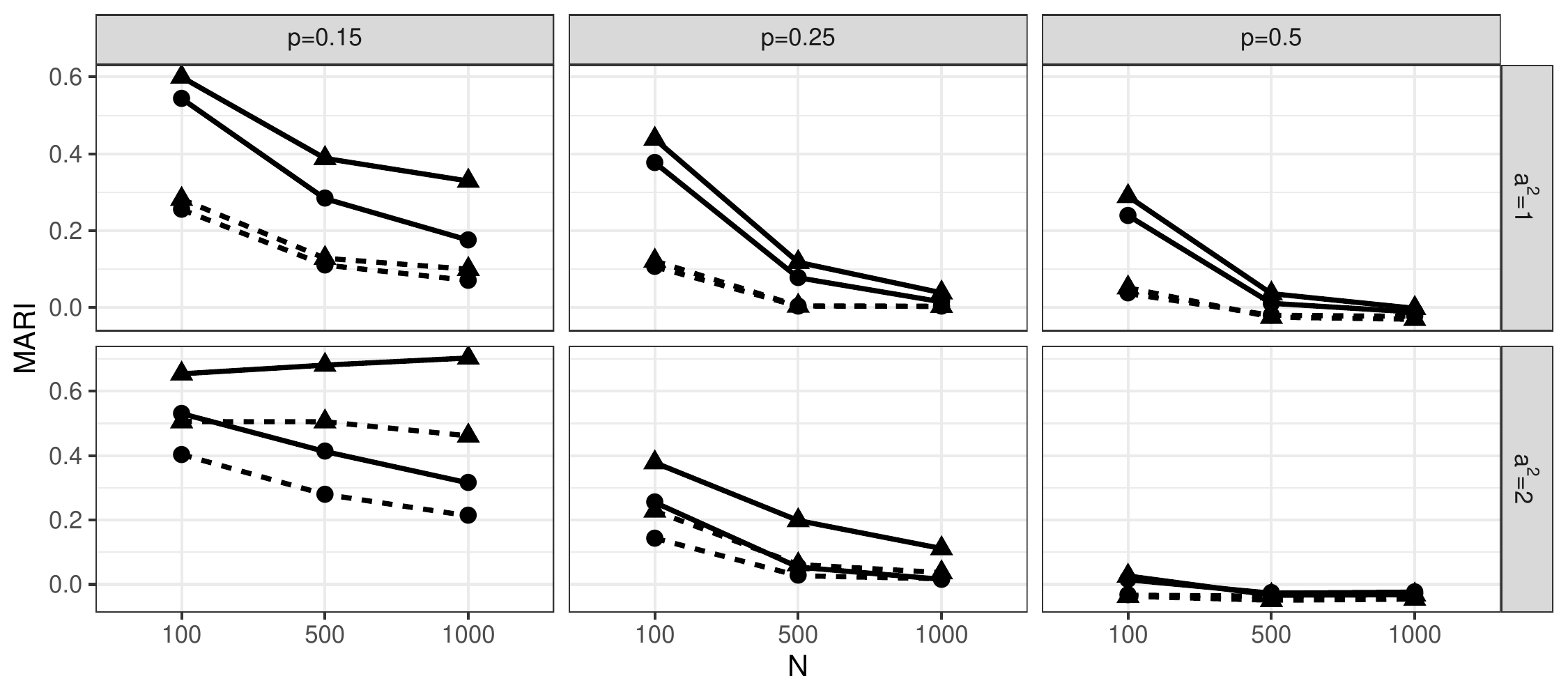}}
\caption{Adjusted rand index based on pMLE and MWDE
when subpopulation distributions are mis-specified}
\label{fig:robust2}
\end{figure}

We obtain the MARI values based on $R=1000$ repetitions with 
sample sizes $N=100$, $500$, and $1000$, see Figure~\ref{fig:robust1} and Figure~\ref{fig:robust2}.
We see that when the degree of overlap is low, MWDE and pMLE have similar performances. When the subpopulation variance is larger ($a^2 = 2$), the performance of pMLE is generally better.  In general, we conclude
that pMLE is preferred.

Statistical inference usually becomes more accurate when the sample size increases. 
This is not the case in this simulation experiment. We can see that MARI often decreases (becomes less accurate) when the sample size increases.
This is not caused by simulation error. 
When the model is mis-specified, the learned model does not converge to the "true model" as $N \to \infty$. 
Hence, the inference does not necessarily improve.
The moral of this simulation study is that the MWDE is not more robust
than the pMLE, against our intuition.

\subsection{Image Segmentation}
Image segmentation aims to partition an image into regions,
each with a reasonably homogeneous visual appearance 
or corresponds to objects or parts of objects~\citep[Chapter 9]{bishop2006pattern}. 
In this section, we perform image segmentation with finite normal mixtures,
a common practice in the machine learning community.

Each pixel in an image is represented by three numbers within the 
range of $[0,1]$ that corresponds to the intensities of the Red, Green, and Blue (RGB) channels.
Since the intensities values are always between 0 and 1,
unlike the common practice in the literature, we feel obliged to
transform the intensity values to ensure the normal mixture model fits better.
Let $y = \Phi^{-1}((x+1/N)/(1+2/N))$ with $x$ being the intensity 
and $N$ the total number of pixels in the image. 
We then learn a two-component normal mixture on $y$ values from each channel. 
Namely, we learn three normal mixtures on red, green, and blue channels respectively. 

We use the maximum posterior probability rule to assign each pixel to one of two clusters. We then form an image segment by pixels assigned to the same cluster. 
We visualize the segregated images channel-by-channel by re-drawing the image  with the original intensity value replaced by the average intensity of the pixels assigned to the specific cluster.

The segregated images depend heavily on the fitted mixture distributions.
We compare the segregated images obtained by the normal mixtures learned via the pMLE and MWDE.
We retrieved an image from Pexel~\footnote{~\url{https://www.pinterest.se/pin/761952830692007143/}}
as shown in Figure~\ref{fig:image_segmentation} (a).
\cite{clark2015pillow} resized the original high-resolution image to $433 \times 650$ grids using Lanczos filter.
We learn a normal mixture of order $K=2$ for each channel based on resized data sets and evaluated its utility of segregating the foreground and the background. 
\begin{table}[!ht]
\centering
\small
\caption{Estimated mixing distributions of the flower image by pMLE and MWDE.}
\begin{tabular}{lccccccc}
\toprule
 Channel &  Estimator& $w_1$ & $w_2$ & $\mu_1$ & $\mu_2$ & $\sigma_1$ & $\sigma_2$ \\
    \cmidrule(lr){1-8} 
\multirow{2}{*}{Red}& pMLE&0.896&0.104&-1.668&1.139&1.321&0.277\\
& MWDE&0.915&0.085&-1.617&1.220&1.316&0.213\\
    \cmidrule(lr){2-8} 
\multirow{2}{*}{Green} & pMLE & 0.804 & 0.196 & -0.935 & 0.637 & 0.373 & 0.595\\
&MWDE&0.819&0.181&-0.926&0.724&0.378&0.510\\
    \cmidrule(lr){2-8} 
\multirow{2}{*}{Blue} & pMLE& 0.735 & 0.265 & -0.753 & 0.268 & 0.414 &1.034\\
&MWDE& 0.862&0.138&-0.722&1.019&0.473&0.592\\
\bottomrule
\end{tabular}
\label{tab:image_segmentation}
\end{table}

\begin{figure}[htbp]
    \centering  
\subfloat[]{\includegraphics[width=0.19\linewidth,height=0.19\linewidth]
{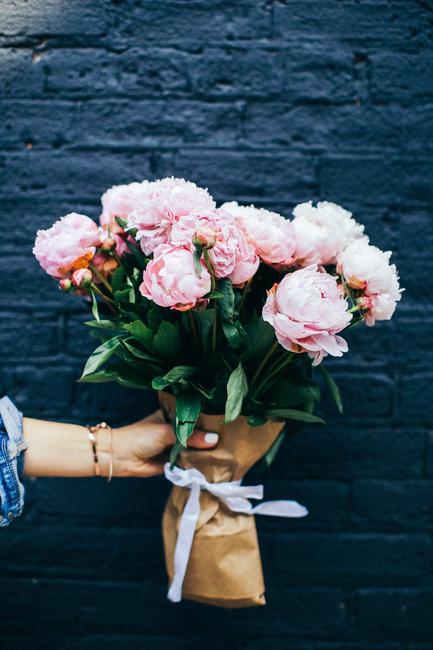}}
\subfloat[]{\includegraphics[width=0.19\linewidth,height=0.19\linewidth]
{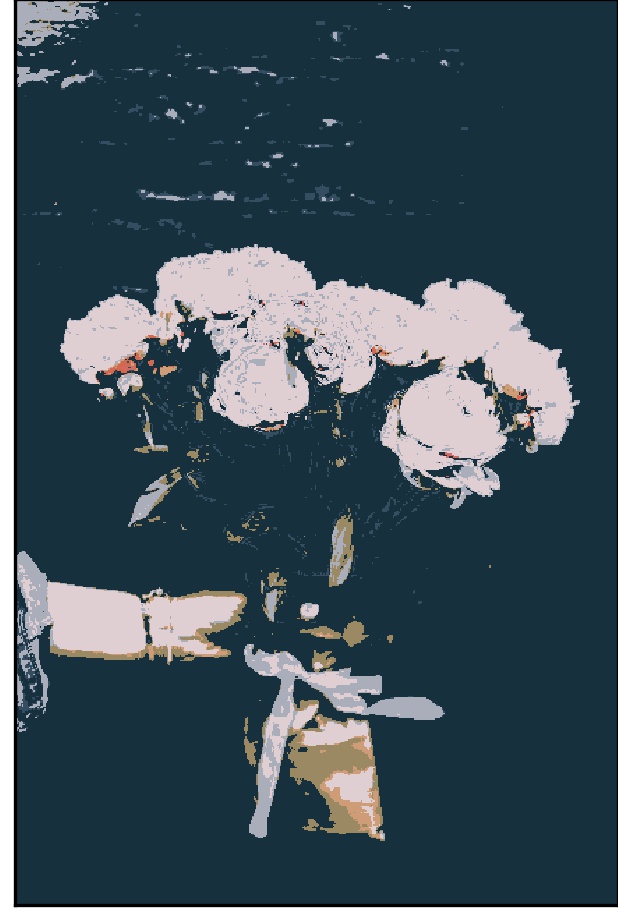}}
\subfloat[]{\includegraphics[width=0.19\linewidth,height=0.19\linewidth]
{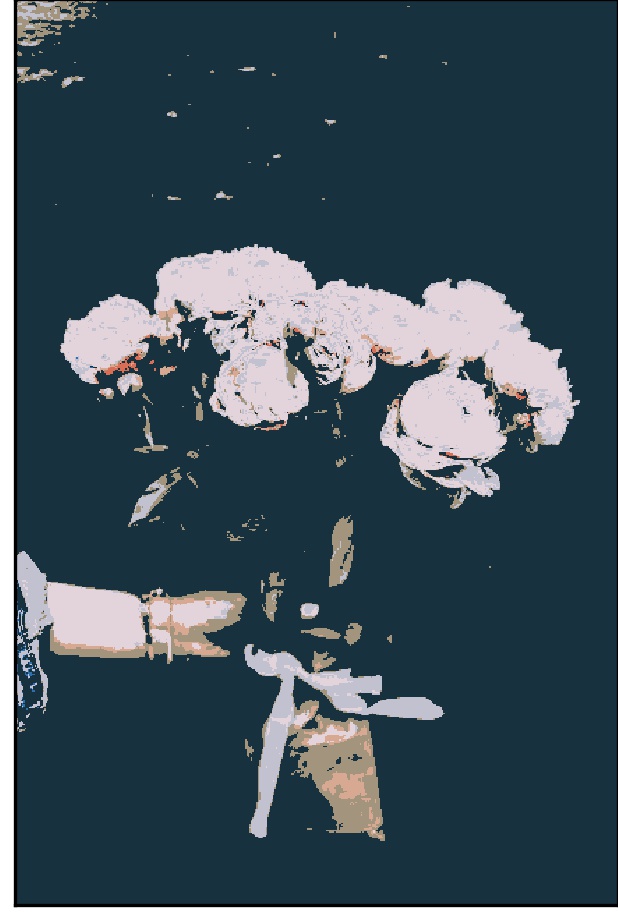}}
\\
\subfloat[]{\includegraphics[width=0.19\linewidth,height=0.19\linewidth]
{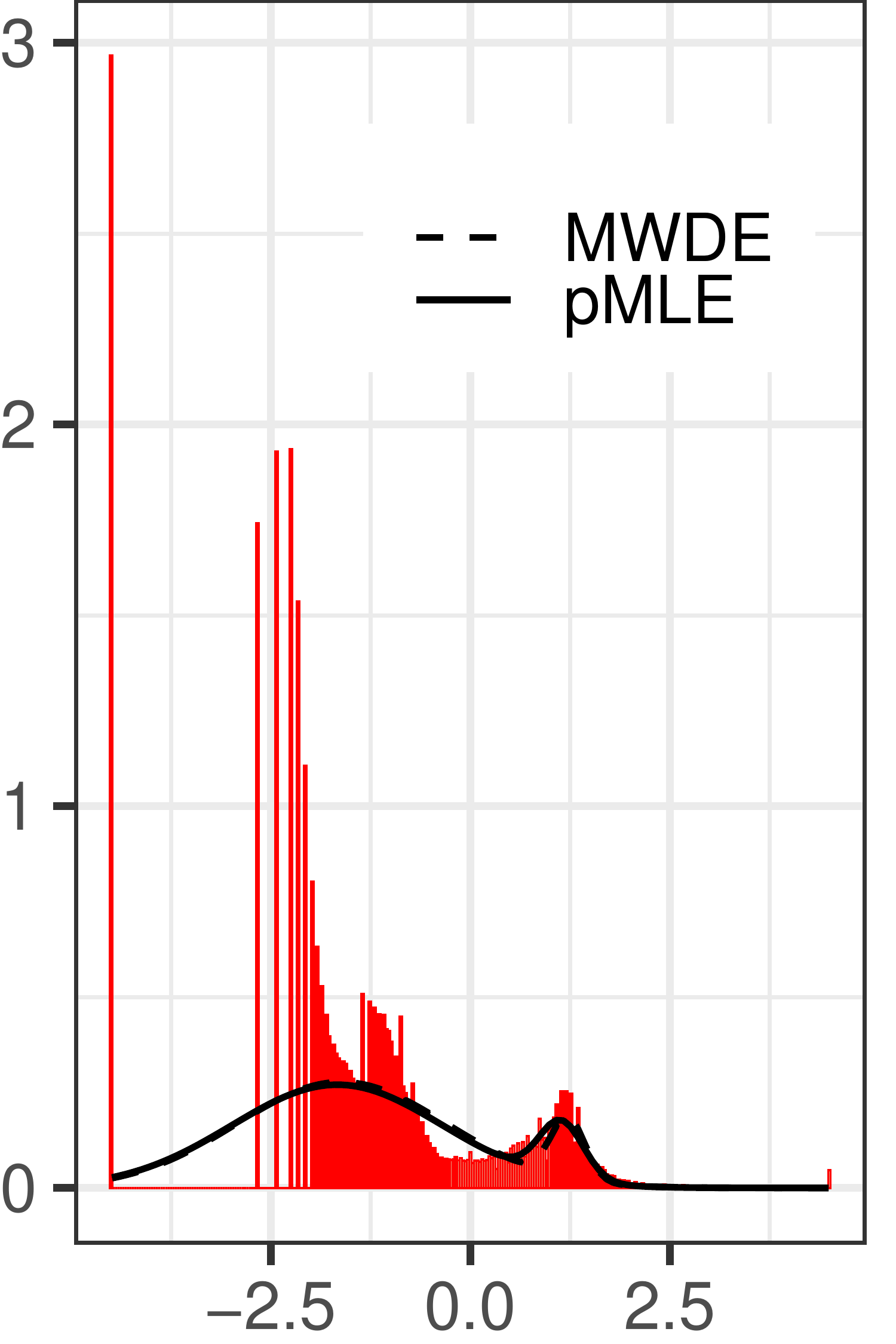}}
\subfloat[]{\includegraphics[width=0.19\linewidth,height=0.19\linewidth]
{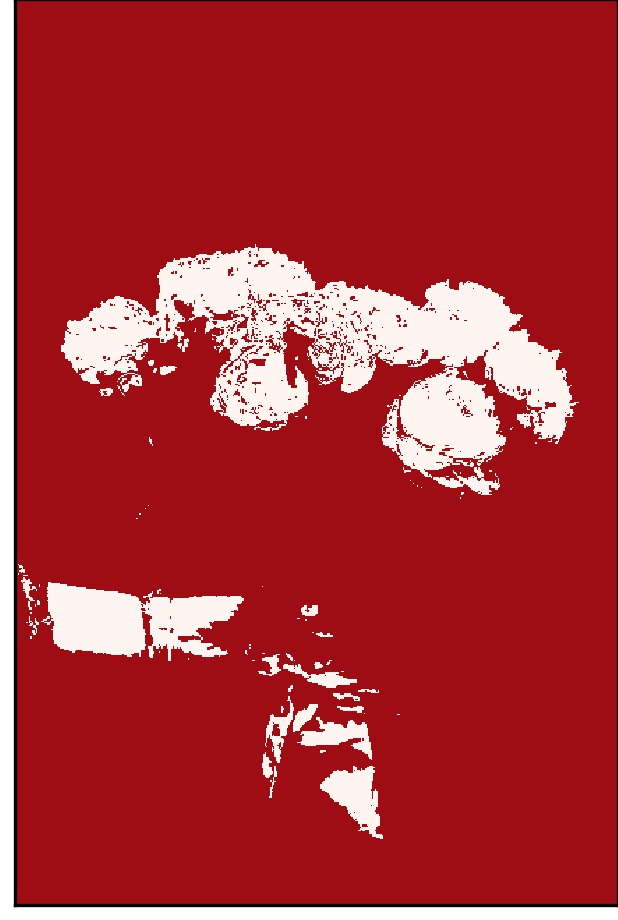}}
\subfloat[]{\includegraphics[width=0.19\linewidth,height=0.19\linewidth]
{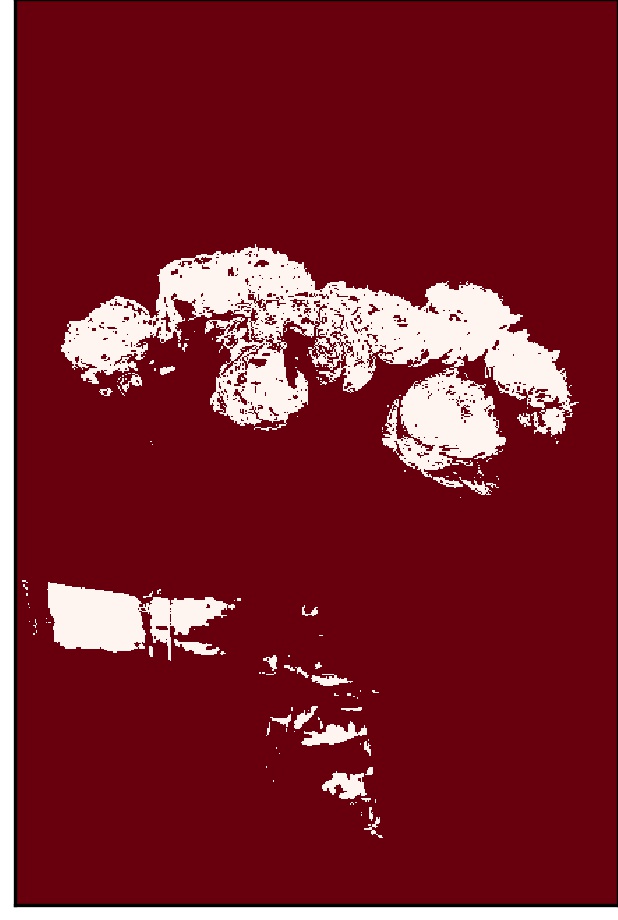}}
\\
\subfloat[]{\includegraphics[width=0.19\linewidth,height=0.19\linewidth]
{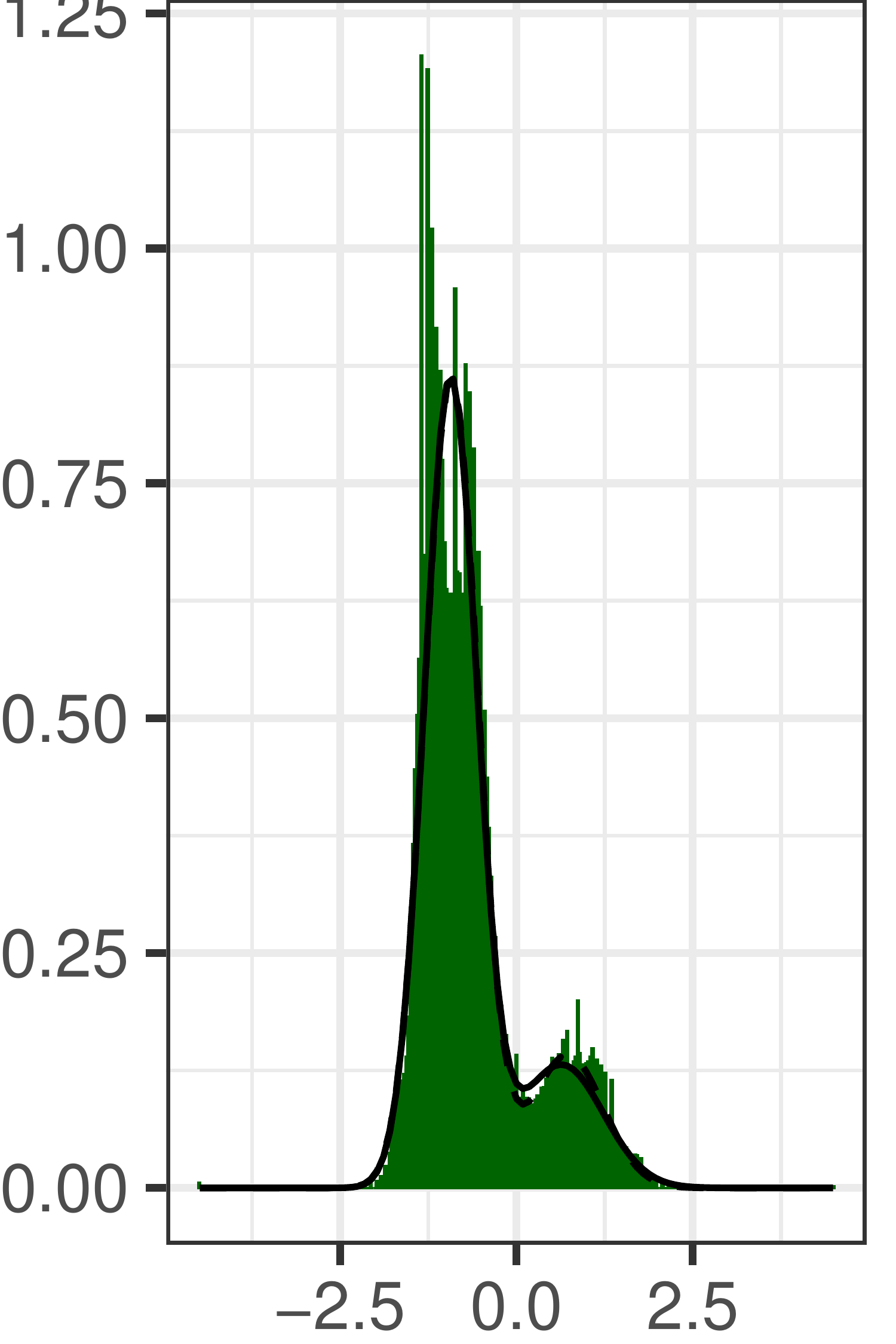}}
\subfloat[]{\includegraphics[width=0.19\linewidth,height=0.19\linewidth]
{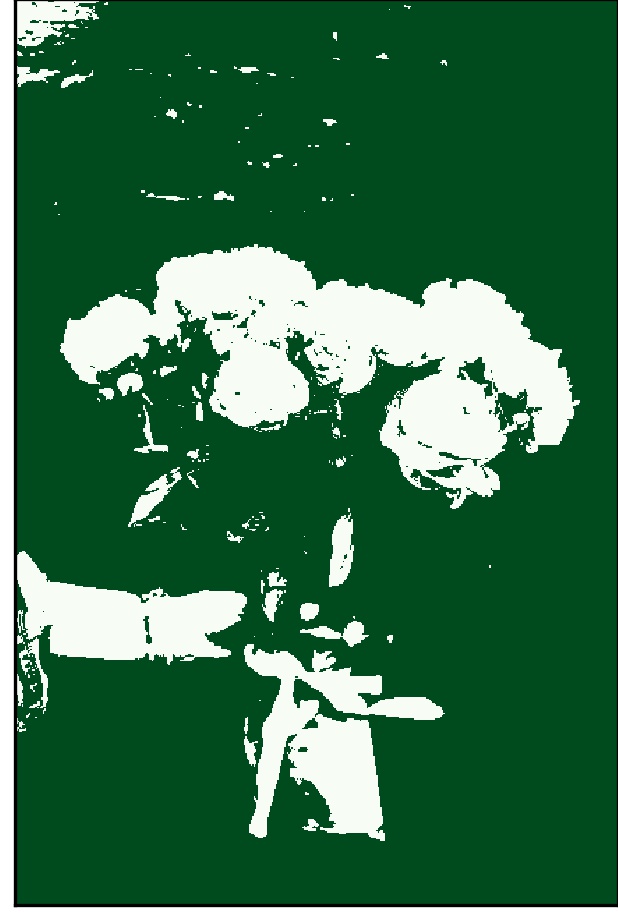}}
\subfloat[]{\includegraphics[width=0.19\linewidth,height=0.19\linewidth]
{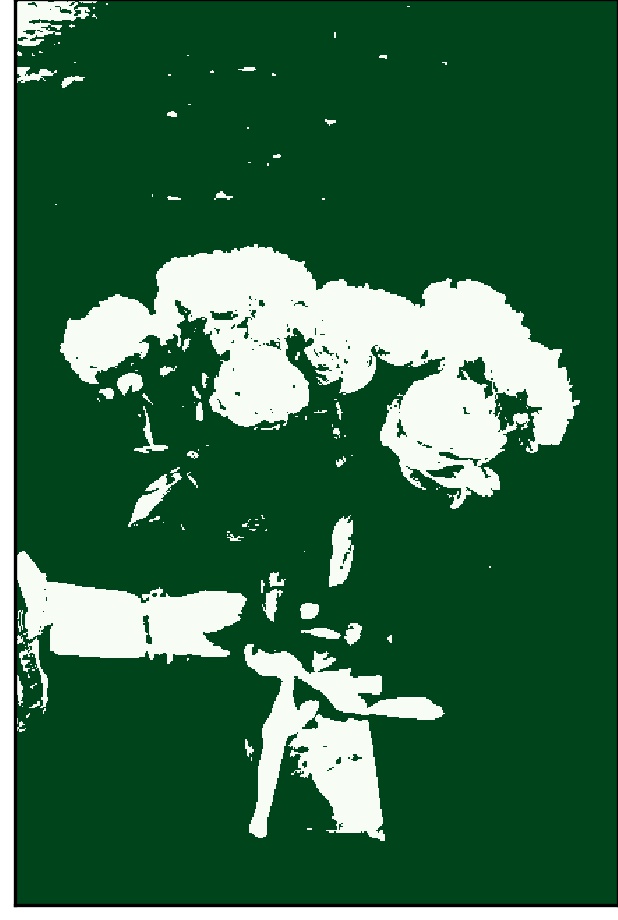}}
\\
\subfloat[]{\includegraphics[width=0.19\linewidth,height=0.19\linewidth]
{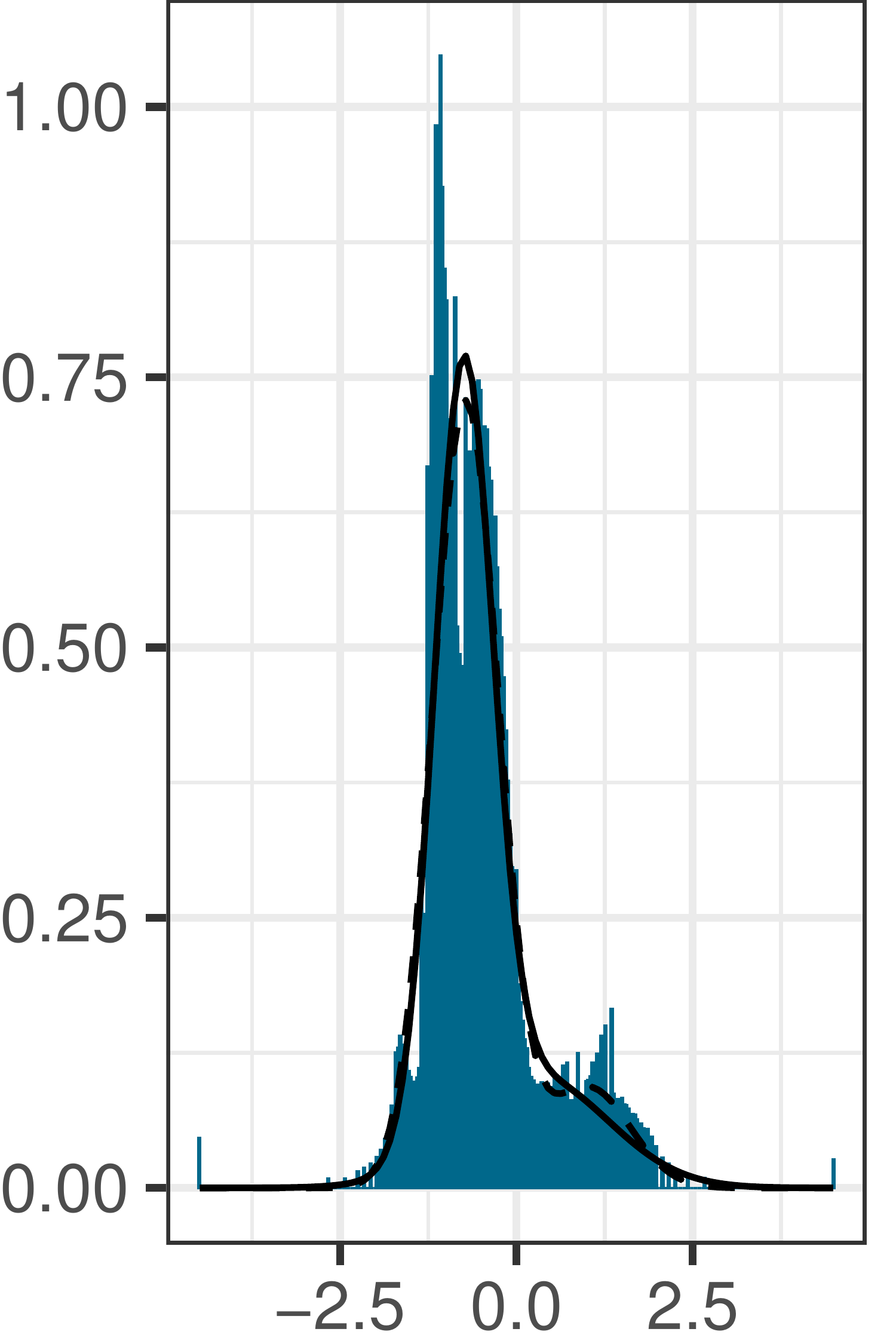}}
\subfloat[]{\includegraphics[width=0.19\linewidth,height=0.19\linewidth]
{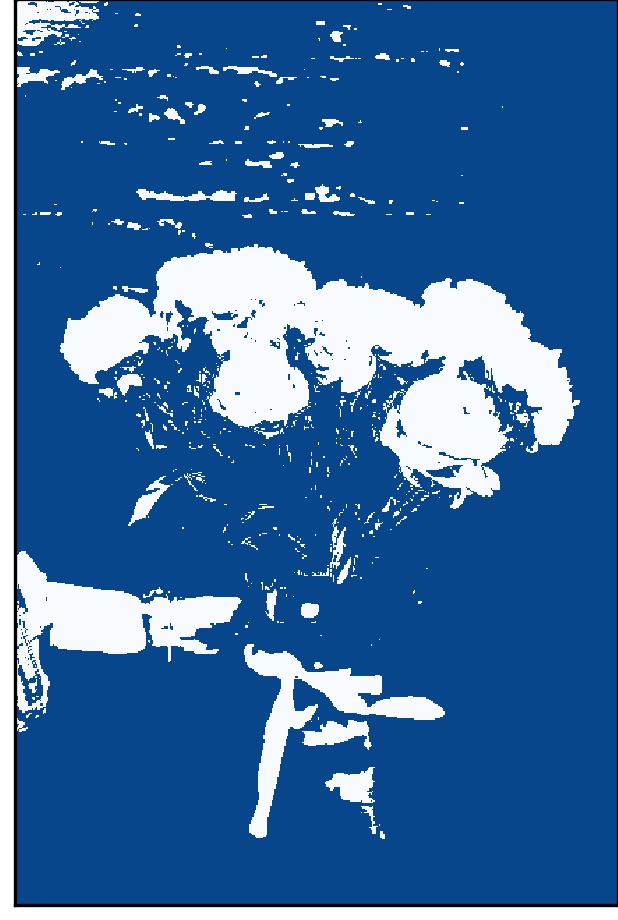}}
\subfloat[]{\includegraphics[width=0.19\linewidth,height=0.19\linewidth]
{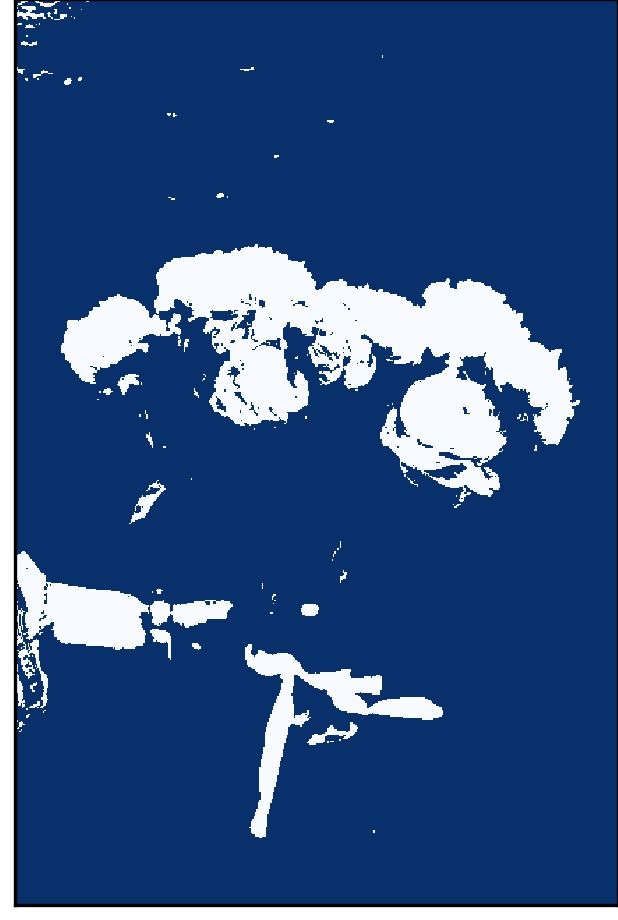}}
\caption{Flower image and its segmentation outcomes. 
}
 \label{fig:image_segmentation}
\end{figure}
We present the specifications of the learned mixing distributions by pMLE and MWDE in Table \ref{tab:image_segmentation}.
Plots (d), (g), and (j) in Figure~\ref{fig:image_segmentation} are histograms of the transformed intensity values of RGB channels, together with the mixture densities learned via pMLE and MWDE.
The corresponding segmented images are shown as plots (e), (h), and (k) for pMLE; (f), (i), and (l) for MWDE. 
The estimated parameter values and the fitted density on the red and green channels based on these two approaches are very similar. 
For the blue channel, the fitted densities and the segmentation results are very similar although the estimated parameter values of the second component are quite different.
Both approaches can produce images with meaningful structures segregating foreground from background.

There are two clusters in each of 3 channels leading to 8 refined clusters.
We may paint each pixel with the average RGB intensity triplet according to these 8 refined clusters. 
The re-created image via pMLE and MWDE respectively are shown in (b) and (c). 
We note these two images are very similar, showing that both learning strategies are effective.

\section{Conclusion}
\label{sec:conclusion}
The MWDE provides another approach for learning finite location-scale mixtures.
We have shown the MWDE is well defined and consistent.
Our moderate scaled simulation study shows it suffers some efficiency loss 
against a penalized version of MLE in general without noticeable gain in robustness.
We gain the knowledge on the benefits and drawbacks of the MWDE under finite location-scale mixtures.
We reaffirm the general superiority of the likelihood based learning strategies even for non-regular models.

{\small
\bibliographystyle{apalike}
\bibliography{biblio}
}


\section*{Appendix}

{\bf Numerically friendly expression of $W_2(F_N, F(\cdot|G))$.}
\label{sec:analytical_form_of_obj}
To learn the finite mixture distribution through MWDE,
we must compute
\begin{equation*}
\mathbb{W}_{N}(G) =W_2^2(F_N(\cdot), F(\cdot|G)) = \int_{0}^{1} \{ F_N^{-1}(t) - F^{-1}(t|G)\} ^2 dt
\end{equation*}
for finite location scale mixture 
\[
F(\cdot |G) 
= \sum_{k=1}^K \pi_k F(\cdot| \btheta_k)
= \sum_{k=1}^K \pi_k \sigma_k^{-1} F_0( (x - \mu_k)/\sigma_k).
\]
We write ${\mathbb E}_k(\cdot)$ as expectation under
distribution $ F(\cdot| \btheta_k)$.
For instance,
\[
\mathbb{E}_k\{X^2\} 
= \mu_k^2 + \sigma_k^2(\mu_0^2+\sigma_0^2)+2\mu_k\sigma_k\mu_0.
\]

Let $I_n = ( (n-1)/N, n/N]$ for $n=1, 2, \ldots, N$ so that
$F^{-1}_N (t) = x_{(n)}$ when $t \in I_n$,
where $x_{(n)}$ is the $n$th order statistic.
For ease of notation, we write $x_{(n)}$ as $x_n$. 
Over this interval, we have
\begin{equation}
\label{eqn.A1}
\int_{I_n} \{ F^{-1}_N(t) - F^{-1}(t|G) \}^2 dt
=
\int_{I_n} [ x^2_n - 2 x_n F^{-1}(t|G)  + \{ F^{-1}(t|G) \}^2 ] dt.
\end{equation}
The integration of the first term in \eqref{eqn.A1}, after summing over $n$, is 
given by 
\[
\sum_{n=1}^N \int_{I_n} x_n^2 dt =N^{-1} \sum_n x_n^2=\overline{x^2}.
\]
The integration of the third term in \eqref{eqn.A1} is
\[
\sum_{n=1}^N \int_{I_n}  \{ F^{-1}(t|G) \}^2 dt 
=
\int_{-\infty}^{\infty} x^2 f(x|G) dx
=
\sum_{k=1}^{K} w_k \mathbb{E}_k \{X^2\}.
\]
Let $\xi_0=-\infty$, $\xi_{N+1}=\infty$, and $\xi_n = F^{-1}(n/N|G)$ for $n=1, \ldots, N$. 
Denote 
\[
\Delta F_{nk} = F(\xi_{n}| \btheta_k) -  F(\xi_{n-1}| \btheta_k)
\]
and
\[
T(x) = \int_{-\infty}^x  t f_0(t) dt, ~~~
\Delta T_{nk}  = T((\xi_{n}-\mu_k)/\sigma_k) - T(\xi_{n-1}-\mu_k)/\sigma_k).
\]
Then
\begin{align*}
\int_{I_n} F^{-1}(t|G)dt
&=
\sum_k w_k \int_{\xi_{n-1}}^{\xi_{n}} x f(x|\mu_k,\sigma_k) dx \\
&=
\sum_k w_k \{ \mu_k   \Delta F_{nk} +  \sigma_k \Delta T_{nk} \}.
\end{align*}
These lead to numerically convenient expression
\[
\mathbb{W}_{N}(G) 
=
\overline{x^2}+
\sum_k w_k {\mathbb E}_{k}\{ X^2\}
-  
2\sum_k w_k \{ \mu_k   \Delta F_{nk} +  \sigma_k \Delta T_{nk} \}.
\]
To most effectively use BFGS algorithm, it is best to provide gradients of the objective function. 
Here are some numerically friendly expressions of some partial derivatives.

\begin{lemma}
\label{lemma:cdf_pd}
Let $\delta_{jk} = 1$ when $j=k$ and $\delta_{jk} = 0$ when $j \neq k$.
For $n = 1,\ldots, N$ and $j=1, 2, \ldots, K$, we have
\begin{align*}
\frac{\partial}{\partial w_j} F(\xi_n| \btheta_k) &= 
	 f(\xi_n|\btheta_k)\frac{\partial \xi_n}{\partial w_j},\\
\frac{\partial}{\partial \mu_j} F(\xi_n| \btheta_k)&=
	f(\xi_n|\btheta_k)\left (\frac{\partial \xi_n}{\partial \mu_j}-\delta_{jk}\right ), \\
\frac{\partial}{\partial \sigma_j} F(\xi_n| \btheta_k) &= 
	f(\xi_n|\btheta_k) \Big (\frac{\partial\xi_n}{\partial\sigma_j} 
		-\left \{\frac{\xi_n-\mu_k}{\sigma_k}\right \}\delta_{jk}\Big ),
\end{align*}
and
\begin{align*}
\frac{\partial}{\partial w_j}T \left (\frac{\xi_n-\mu_k}{\sigma_k} \right ) 
	&= f(\xi_n|\btheta_k) \left ( \frac{\xi_n-\mu_k}{\sigma_k}\right ) \frac{\partial\xi_i}{\partial w_j}, \\
\frac{\partial}{\partial \mu_j}T \left (\frac{\xi_n-\mu_k}{\sigma_k} \right ) 
	& =  f(\xi_n|\btheta_k) \left ( \frac{\xi_n-\mu_k}{\sigma_k}\right )  
		\left ( \frac{\partial\xi_n}{\partial \mu_j}-\delta_{jk}\right ), \\
\frac{\partial}{\partial \sigma_j} T \left (\frac{\xi_n-\mu_k}{\sigma_k} \right ) 
	& =  f(\xi_n|\btheta_k) \left ( \frac{\xi_n-\mu_k}{\sigma_k}\right )  
		\left \{  \frac{\partial\xi_i}{\partial\sigma_j} -
			 \left ( \frac{\xi_n-\mu_k}{\sigma_k} \right ) \delta_{jk} \right \}.
\end{align*}
Furthermore, we have
\begin{align*}
\frac{\partial\xi_n}{\partial \mu_k} &= \frac{w_k f(\xi_i|\btheta_k)}{f(\xi_n|G)}, \\
\frac{\partial\xi_n}{\partial \sigma_k} &= 
	\frac{w_k f(\xi_n|\btheta_k)}{f(\xi_i|G)}\left ( \frac{\xi_n-\mu_k}{\sigma_k}\right ) ,\\
\frac{\partial\xi_n}{\partial w_k} &= - \frac{ F(\xi_n|\btheta_k)}{f(\xi_n|G)}.
\end{align*}
\end{lemma}

Based on this lemma, it is seen that
\begin{align*}
\frac{\partial}{\partial \mu_j}\mathbb{W}_N& =
 2w_j(\mu_j+\sigma_j\mu_0)  - 2w_j\sum_{n=1}^N x_{(n)}\Delta F_{nj} \\
& 
 - 2 \sum_{n=1}^{N}\sum_{k} w_k \mu_k x_{(n)} \left \{ 
 \
 \frac{\partial F_0(\xi_n| \btheta_k)}{\partial\mu_j}
 - \frac{\partial F_0(\xi_{n-1} | \btheta_k)}{\partial\mu_j}
 \right \}
  \\
&
-2\sum_{n=1}^{N}\sum_{k} w_k\sigma_k x_{(n)}
	\frac{\partial }{\partial \mu_j}
		  \left \{
		  T \left ( \frac{\xi_{n}-\mu_k}{\sigma_k} \right ) 
		  -
		   T \left ( \frac{\xi_{n-1}-\mu_k}{\sigma_k} \right )
		  \right \}
\end{align*}
with $F_0(\xi_0|\theta_k)=0$, $F_0(\xi_{N+1}|\theta_k)=1$, 
$T \big ( \frac{\xi_{0}-\mu_k}{\sigma_k} \big)=0$, 
and $T\big(\frac{\xi_{N+1}-\mu_k}{\sigma_k} \big )=\int_{-\infty}^{\infty} tf_0(t)dt$ 
is a constant that does not depend on any parameters.
Substituting the partial derivatives in Lemma~\ref{lemma:cdf_pd}, we then get
\begin{align*}
\frac{\partial}{\partial \mu_j}\mathbb{W}_N
=&
 2w_j(\mu_j+\sigma_j\mu_0) -2w_j\sum_{n=1}^N x_{(n)}\Delta F_{nj}\\
&
-2\sum_{n=1}^{N-1}x_{(n)}\xi_n\sum_{k}w_k f(\xi_n|\mu_k,\sigma_k) 
\big(\frac{\partial\xi_n}{\partial\mu_j}-\delta_{jk}\big)\\
&
+2\sum_{n=1}^{N-1}x_{(n)}\xi_{n-1}\sum_{k}w_k f(\xi_{n-1}|
     \mu_k,\sigma_k) \big(\frac{\partial\xi_{n-1}}{\partial\mu_j}-\delta_{jk}\big)\\
=&
2w_j\big\{\mu_j + \sigma_j\mu_0 - \sum_{n=1}^N x_{(n)}\Delta F_{nj}\big\}
\end{align*}


Similarly, we have
\begin{align*}
\frac{\partial}{\partial \sigma_j}\mathbb{W}_N
=& \ 
2w_j\{\sigma_j(\mu_0^2+\sigma_0^2) + \mu_j\mu_0-\sum_{n=1}^N x_{(n)}\Delta\mu_{nj}\},
\\
\frac{\partial}{\partial w_k} \mathbb{W}_N
=&\
 \{\mu_k^2 + \sigma_k^2(\mu_0^2+\sigma_0^2)+2\mu_k\sigma_k\mu_0\} 
- 2 \sum_{n=1}^{N-1} \{x_{(n+1)} - x_{(n)}\} \xi_iF(\xi_n|\btheta_k)\\
&
- 2\big \{
\mu_k \sum_{n=1}^{N} x_{(n)} \Delta F_{nk}
+
\sigma_k \sum_{n=1}^{N} x_{(n)}\Delta T_{nk}
\big \}.
\end{align*}

Computing the quantiles of the mixture distribution $F(\cdot|G)$ for each $G$
is one of the most demanding tasks. 
The property stated in the following lemma allows us to develop a bi-section algorithm. 

\begin{lemma}
Let $F(x| G)=\sum_{k=1}^K F(x|\mu_k,\sigma_k)$ be a $K$-component mixture, 
$\xi(t) = F^{-1}(t|G)$ and $\xi_k (t) = F^{-1}(t | \btheta_k)$ 
respectively the $t$-quantile of the mixture and its $k$th subpopulation. 
For any $t \in ( 0,1) $,
\begin{equation}
\label{eq:qmixnorm_bound}
\min_{k} \xi_{k}(t)\leq \xi(t) \leq \max_{k} \xi_{k}(t).
\end{equation}
\end{lemma}

\begin{proof}
Since $F(x | \btheta)$ has a continuous CDF,
we must have $F(\xi_{k}(t)| \btheta_k) = t$. 
By the monotonicity of the CDF $F(\cdot| \btheta_k)$, we have
\begin{equation*}
\label{eq:qmixnorm}
F(\min_{k}\xi_{k}(t)| \btheta_k) \leq F(\xi_{k}(t)| \btheta_k)
\leq 
F(\max_{k} \xi_{k}(t)| \btheta_k ).
\end{equation*}
Multiplying by $w_k$ and summing over $k$ lead to
\[
F(\min_{k}\xi_{k}(t)| G)\leq t\leq F(\max_{k} \xi_{k}(t)| G).
\]
This implies~\eqref{eq:qmixnorm_bound} and completes the proof. 
\end{proof}

In view of this lemma, we can easily find the quantiles of $F(\cdot | \btheta_k)$
to form an interval containing the targeting quantile of $F(\cdot| G)$.
We can quickly find $F^{-1}(t | G)$ value through a bi-section algorithm.

\end{document}